\documentclass{article} 
\usepackage{iclr2025_conference,times}

\usepackage{amsmath,amsfonts,bm}

\def\eqref#1{equation~\ref{#1}}

\def\1{\bm{1}}

\DeclareMathAlphabet{\mathsfit}{\encodingdefault}{\sfdefault}{m}{sl}
\SetMathAlphabet{\mathsfit}{bold}{\encodingdefault}{\sfdefault}{bx}{n}

\newcommand{\E}{\mathbb{E}}

\usepackage{hyperref}
\usepackage{url}
\usepackage{cleveref}
\usepackage{wrapfig}
\usepackage{multirow}
\usepackage{booktabs}
\usepackage{makecell}

\usepackage[utf8]{inputenc} 
\usepackage[T1]{fontenc}    
\usepackage{hyperref}       
\usepackage{url}            
\usepackage{booktabs}       
\usepackage{amsfonts}       
\usepackage{nicefrac}       
\usepackage{microtype}      
\usepackage{standalone}
\usepackage{latexsym}
\usepackage{amsmath}
\usepackage{amssymb}
\usepackage{amsthm}
\usepackage{natbib}   
\usepackage{graphicx}
\usepackage{subcaption}
\usepackage{array}
\usepackage{tabu}
\usepackage{makecell}
\usepackage{paralist}
\usepackage{cases}
\usepackage{diagbox}
\usepackage{enumitem}
\usepackage{soul}
\usepackage{multirow}
\usepackage{verbatim}
\usepackage{tabulary}
\usepackage{booktabs}
\usepackage{tabularx}
\usepackage[mathscr]{euscript}
\usepackage{mathtools}
\usepackage{stmaryrd}
\usepackage{tikz-dependency}
\usetikzlibrary{automata,decorations.markings,arrows,positioning,matrix,calc,patterns,angles,quotes,calc}
\usepackage{adjustbox}
\usepackage{tabularx}
\usepackage{xspace}
\usepackage{tabulary}
\usepackage{afterpage}
\usepackage{bm}
\usepackage{color}
\usepackage{graphicx}
\usepackage{slashbox}
\usepackage[toc,page]{appendix}
\usepackage{makecell}
\usepackage{boldline}
\usepackage[shortcuts]{extdash}  

\usepackage{blindtext}
\usepackage{graphicx}
\usepackage{capt-of}
\usepackage{booktabs}
\usepackage{varwidth}
\usepackage{pifont}
\usepackage{wrapfig}

\usepackage{listings}

\usepackage{pythonhighlight}

\hypersetup{
	colorlinks=true,
	linkcolor=cyan,
	filecolor=blue,      
	urlcolor=purple,
	citecolor=cyan,
}

\theoremstyle{plain}
\newtheorem{theorem}{Theorem}[section]
\newtheorem{proposition}[theorem]{Proposition}

\theoremstyle{definition}

\theoremstyle{remark}

\NewDocumentCommand{\lifan}
{ mO{} }{\textcolor{cyan}{\textsuperscript{\textit{Lifan}}\textsf{\textbf{\small[#1]}}}}
\NewDocumentCommand{\huayu}
{ mO{} }{\textcolor{blue}{\textsuperscript{\textit{chy}}\textsf{\textbf{\small[#1]}}}}
\NewDocumentCommand{\hao}
{ mO{} }{\textcolor{purple}{\textsuperscript{\textit{hao}}\textsf{\textbf{\small[#1]}}}}
\NewDocumentCommand{\ning}
{ mO{} }{\textcolor{orange}{\textsuperscript{\textit{ning}}\textsf{\textbf{\small[#1]}}}}
\NewDocumentCommand{\ganqu}
{ mO{} }{\textcolor{brown}{\textsuperscript{\textit{ganqu}}\textsf{\textbf{\small[#1]}}}}

\title{Free Process Rewards without Process Labels}

\author{Lifan Yuan$^{1}$\thanks{\hspace{1mm}Equal Contribution. Work done during Wendi's intership at Tsinghua University.} \quad 
Wendi Li$^{2,3*}$ \quad
Huayu Chen$^{2}$ \quad
Ganqu Cui$^{2}$\thanks{\hspace{1mm}Corresponding Authors: \texttt{cgq22@mails.tsinghua.edu.cn, dn97@mail.tsinghua.edu.cn}} \quad
Ning Ding$^2$\footnotemark[2]  \quad
Kaiyan Zhang$^2$ \quad\\
\textbf{Bowen Zhou$^2$} \quad
\textbf{Zhiyuan Liu$^2$} \quad
\textbf{Hao Peng$^1$}\\
$^1$University of Illinois Urbana-Champaign \quad
$^2$Tsinghua University \\
$^3$Huazhong University of Science and Technology \\
\texttt{lifan4@illinois.edu} \quad
\texttt{wendili@hust.edu.cn}
}

\iclrfinalcopy 
\begin{document}

\maketitle

\begin{abstract}
\looseness=-1
Different from its counterpart outcome reward models (ORMs), which evaluate the entire responses, a process reward model (PRM) scores a reasoning trajectory step by step, providing denser and more fine grained rewards. However, training a PRM requires labels annotated at every intermediate step, presenting significant challenges for both manual and automatic data collection. This paper aims to address this challenge. Both theoretically and empirically, we show that an \textit{implicit PRM} can be obtained \textit{at no additional cost}, by simply training an ORM on the cheaper \textit{response-level labels}. The only assumption is to parameterize the outcome reward as the log-likelihood ratios of the policy and reference models $r_\theta (\mathbf{y})=\beta \log \frac{\pi_\theta(\mathbf{y})}{\pi_{\text{ref}}(\mathbf{y})}$, which can be optimized regardless of the specific choice of loss objectives. In experiments, we instantiate our implicit PRMs with various objectives and evaluate their performance on MATH. We show that our implicit PRM outperforms a strong MCTS-based baseline \textit{\'a la} Math-Shepherd~\citep{Wang2023MathShepherdVA} using less than $1/38$ of the training data. Its performance can be further improved  with majority voting. We further find that scaling up instructions and responses benefits our implicit PRM, and the latter brings a larger gain. Particularly, we find that our implicit PRM, when instantiated with the cross-entropy (CE) loss, is more data-efficient and can keep improving generation models even when trained with only one response per instruction, the setup that suffers from extreme data scarcity and imbalance. Further, instructions should be relevant to downstream tasks while the diversity of responses does not bring gains. Surprisingly, training on extra Math-Shepherd step labels brings no further improvements to our implicit PRM trained on only outcome data. We hope that our work will encourage a rethinking of PRM training approaches and contribute to making training PRMs more accessible\footnote{Models and data are available at: \url{https://github.com/lifan-yuan/ImplicitPRM}.}.

\end{abstract}

\begin{figure}[bth]
    \centering
    \vspace{-18pt}
    \includegraphics[width=0.75\textwidth]{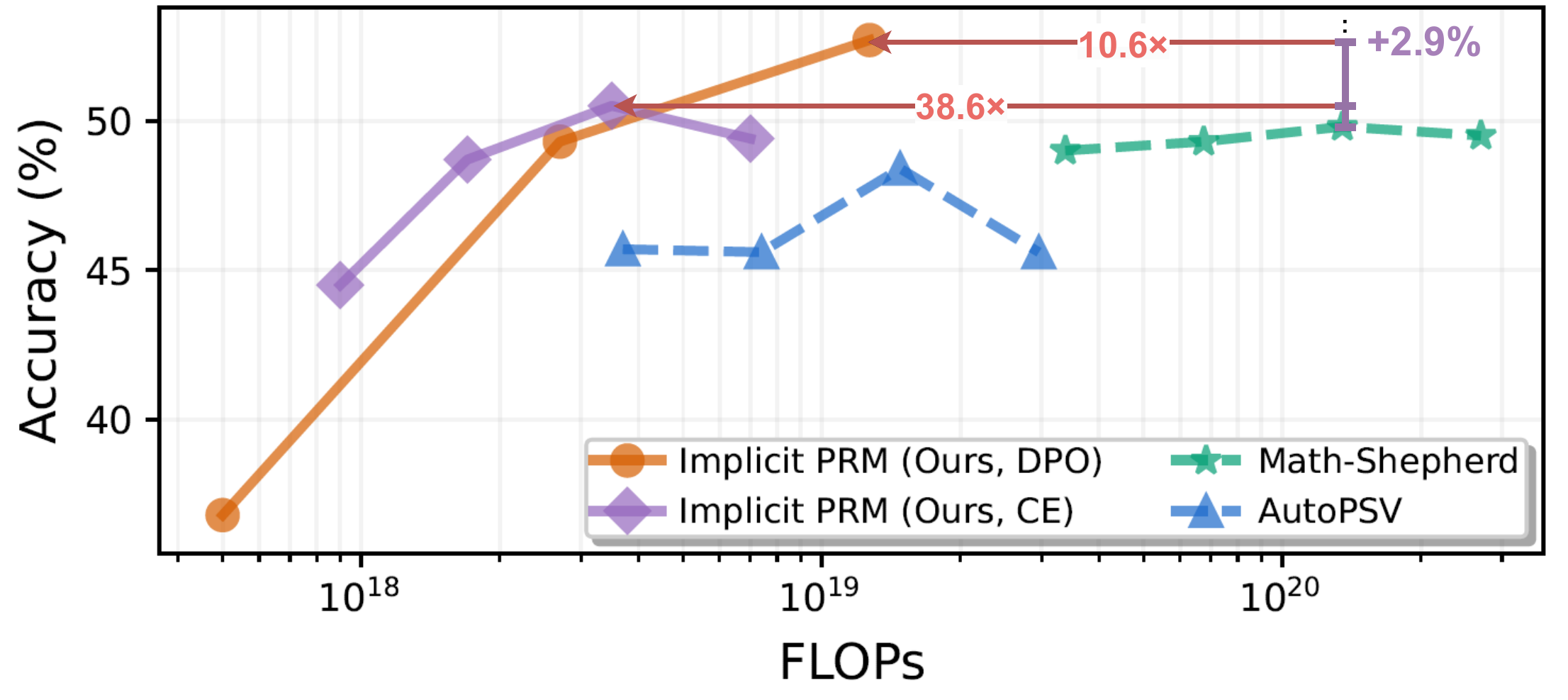}
    \vspace{-10pt}
    \caption{The x-axis indicates the FLOPs required to collect the data and train the model,
    and y axis the accuracies of best-of-64 performance.
    The accuracy is averaged over the best-of-64 accuracies of Mistral-7B-Instruct-v0.2 \citep{Jiang2023Mistral7}, Llama-3.1-8B-Instruct, and Llama-3.1-70B-Instruct \citep{llama3modelcard} on MATH \citep{Hendrycks2021MeasuringMP}.
    Different dots on the same line indicates models trained with the same approach but on different scales of data.
    The top-left zone is desirable in this figure, as it suggests a model can achieve higher performance with less development overhead.
    Our implicit PRM is much cheaper to train while presenting the best performance under the same budget.
    }
    \label{fig:token-acc}
\end{figure}

\section{Introduction}

Training on high-quality supervised data has driven the advances in LLMs development  \citep{ llama3modelcard, ding2023enhancing, Luo2023WizardCoderEC, Yue2024MAmmoTH2SI, Yuan2024AdvancingLR, zhang2024ultramedical}.
Building upon this progress, reward models push the boundaries even further, especially in tasks requiring complex reasoning~\citep{Lightman2023LetsVS, Wang2023MathShepherdVA, Snell2024ScalingLT}.
Outcome Reward Models (ORMs), designed to evaluate full responses,  have been primarily explored, which can be used in both reinforcement learning (RL) and inference. 
However, due to the sparsity of outcome rewards, ORMs often yield suboptimal performance when reranking responses at inference~\citep{Lightman2023LetsVS} and struggle with stability and efficiency during RL training~\citep{Cao2024EnhancingRL, Chan2024DenseRF}.
This highlights the growing demand for \textit{denser and more fine-grained rewards}.
Process Reward Models (PRMs), evaluating intermediate steps to provide fine-grained guidance, naturally meet this need.
Existing work has shown consistent results that PRMs outperform ORMs in best-of-N sampling \citep{Wang2023MathShepherdVA, Snell2024ScalingLT} and RL \citep{Setlur2024RewardingPS}, and argues that scoring every intermediate step provides better transparency and interpretability \citep{Leike2024tweet}.
\looseness=-1
Despite their promise, PRMs are much harder to train than ORMs, since collecting PRM training data requires annotating every intermediate step.
To reduce human efforts, automatic annotation approaches have been proposed, where an intermediate step is labeled based on its estimated probability of leading to a correct outcome.
Typically, this is achieved through either sampling massive look-ahead trajectories to estimate or directly training a verifier to predict Q value, both incurring extensive overhead \citep{Wang2023MathShepherdVA, Lu2024AutoCVER}. 
For example, collecting step-level data through sampling look-ahead trajectories as \cite{Wang2023MathShepherdVA}  requires 38.8$\times$ more FLOPs than training an ORM (\S \ref{sec:main_exp}).

We argue, from both theoretical and empirical perspectives, that building PRMs can be substantially cheaper than previously realized: 
\textbf{a strong PRM can be obtained at no additional cost from training an ORM on the cheaper response-level data,  with a simple reward parameterization.}
Specifically, by parameterizin the reward as the log-likelihood ratio of the policy and the reference models $r_\theta(\mathbf{y})=\beta \log \frac{\pi_\theta(\mathbf{y})}{\pi_{\text{ref}}(\mathbf{y})}$, a common practice in DPO \citep{Rafailov2023DirectPO} and many of its variants \citep{Azar2023IPO, Ethayarajh2024KTOMA, Chen2024NoiseCA, Rosset2024DirectNO, Wu2024SelfPlayPO}, a PRM can be automatically learned during ORM training.
The process reward is then the same log-likelihood ratio, but calculated over a partial response. 
We dub our approach an \textbf{implicit PRM} since it only requries response-level data and ORM training.
Moreover, our insights are agnostic to the specific choice of the training objective,
and are applicable to both DPO and all the variants that adopt the same form of implicit reward;
it further extends to other objectives 
like the Cross-Entropy (CE) loss.
This fresh theoretical insight generalizes the conclusion from \cite{Rafailov2024FromT} that DPO training enables the model to learn the Q function;
practically, our approach is particularly well-suited for scenarios where pairwise data is hard to obtain and algorithms like CE loss remain equally applicable, as shown in \S \ref{sec:scaling}.

\looseness=-1
In experiments, we  train our implicit PRMs 
on a dataset consisting of 33K math instructions and eight solutions for each, and evaluate them through the best-of-N sampling on MATH \citep{Hendrycks2021MeasuringMP}. 
We explore variants of our implicit PRMs instantiated with different training objectives, including DPO, KTO, NCA, and CE.
All produce strong PRMs, outperforming competitive baselines including
our reimplementations of Math-Shepherd \citep{Wang2023MathShepherdVA} and AutoPSV \citep{Lu2024AutoCVER} and six off-the-shelf open ORMs and PRMs, with substantially better trade-offs between accuracy and development overhead, as shown in Figure \ref{fig:token-acc}.
Particularly, when integrated into weighted best-of-N, CE stands as the most effective. 
This makes CE loss appealing in scenarios where pairwise data is hard to collect, since it can handle unpaired and imbalanced data, and is demonstrated to be less data-consuming than DPO in order for an implicit PRM with decent performance.
Further, we find out that our implicit PRM benefits from increased training data, with the scale of responses being more influential than that of instructions.
However, instructions should be relevant to downstream tasks while the diversity of responses does not matter much according to our observations.
Surprisingly, training on step-level data brings no further improvements to our implicit PRMs.
Additionally, despite our implicit PRM remains a language model, its ability to help best-of-N sampling does not translate into its performance on downstream tasks as a policy model. Rather, our worst performing implicit PRM, instantiated by KTO, stands as the only experiencing with an improvement in policy performance.
Finally, we observe that, at least for the models and tasks we consider,  the reference model can be omitted from our implicit PRM, improving the inference efficiency without hurting the accuracy.

Bypassing the need for step labels, our findings substantially lower the data collection and training overhead of building PRMs while delivering stronger performance than existing methods.
We hope that our work will encourage a rethinking of PRM training approaches and contribute to making training PRMs more accessible.

\section{ORMs vs. PRMs: Dilemma of Performance and Expense}
\label{sec:bg_dilemma}
\paragraph{Background}
ORMs assign sparse rewards $r_\theta (\mathbf{y})$ to the entire response, and no feedback is provided until the last token is generated.
In contrast, a PRM assesses the quality of every intermediate step and can provide reward after completing each \citep{Lightman2023LetsVS}.
Given  an instruction and an $n$-step response $\mathbf{y}$  with $y_t$ being the $t$-th step and $\mathbf{y}_{<t}$ being the first $t-1$ steps,
a PRM receives the concatenation of the instruction and the first $t-1$ steps, and assigns a reward to the $t$-th: $r_\theta^t ( \mathbf{y}_{<t}, y_t )$.
The Q value $q_\theta^t( \mathbf{y}_{<t}, y_t )$ indicates the expectation of outcome reward $r_\theta$ conditioned on the observed response $\mathbf{y}_{<t}$ and current step $y_t$. 
\cite{Lightman2023LetsVS} define the process reward as the correctness of each step, while \cite{Wang2023MathShepherdVA} directly consider Q values as process rewards.
We follow \citet{Lu2024AutoCVER} and define process reward as advantages, namely the difference between Q values: $r_\theta^t:= q_\theta^t - q_\theta^{t-1}$.
The benefits of adopting advantages as process rewards have been discussed by concurrent work \citep{Setlur2024RewardingPS}.

\paragraph{PRMs outperformans ORMs in both training and inference 
} 
Both ORMs and PRMs can provide rewards to assess model outputs.
The dense step-level rewards from PRMs lead to stable and effective RL training \citep{Cao2024EnhancingRL, Chan2024DenseRF}, and performs better on reranking responses, with better transparency and interpretability.
Also, ORMs are trained on complete responses, but the value model initialized from it only receives incomplete responses during RL training. On the contrary, PRMs are intrinsically trained to provide dense rewards given partial responses, thus the resulting value models may mitigate out-of-distribution issues that ORMs encounter.

\paragraph{Training PRMs is substantially more expensive than ORMs}
\looseness=-1
Despite its effectiveness, training PRMs is more difficult due to challenges in training data collection. 
To collect training data for PRMs, MCTS is commonly used for automatic step annotation \citep{Wang2023MathShepherdVA, Luo2024ImproveMR}. However, it introduces substantial extra cost.
For MCTS-based step label annotation, a policy model will sample $N$ trajectories based on the concatenation of an instruction $x$ and partial response up to step $t$, each leading to a final answer \citep{Wang2023MathShepherdVA}.
E.g., assuming 10-step rollouts and 8 subsequent trajectories for each step as in \citet{Wang2023MathShepherdVA}, a total of $10\times8=80$ trajectories need to be generated to get step labels for each instruction, which is 80 times more than ORMs.
Therefore, the scaling of PRMs is largely limited.
Besides the overhead of training data collection, this MCTS approach 
can lead to suboptimal performance due to the noisy annotation process, as we will show below and in the experiments.

\paragraph{MCTS estimation is not precise either}
We denote the set of correctness of subsequent trajectories as $\left \{ c_1, c_2, \dots, c_N \right \}$, each element being 0 or 1. Thereafter, two alternative label estimation strategies are available: (1) \textbf{Hard Estimation,} where step $t$ will be labeled as 1 if any rollout is correct and 0 otherwise: $l_t = \max \left \{ c_1, c_2, \dots, c_N \right \}$. (2) \textbf{Soft Estimation,} where step $t$ is labeled as the proportion of correct answers among all rollouts, namely $l_t = \sum_{t=1}^{N} c_t / N$. 
We refer the ORM used to judge the correctness of rollouts as $\theta$, the PRM trained on data from hard estimation as $\theta_{h}$, and the PRM trained on soft estimation data as $\theta_{s}$.
If $\theta_{h}$ and $\theta_{s}$ are perfectly fitted, namely training losses reduced to 0, we have
\begin{equation}
\label{eq:mcts_q}
q_{\theta_{h}}^t\left(\mathbf{y}_{<t}, y_t\right) = \max_{\mathbf{y} \mid \mathbf{y}_{<t}} r_{\theta}(\mathbf{y}), ~~
q_{\theta_{s}}^t\left(\mathbf{y}_{<t}, y_t\right) = \E_{\pi_\text{ref}(\mathbf{y}\mid\mathbf{y}_{\leq t})} r_{\theta}(\mathbf{y})
\end{equation}
However, both estimation strategies may be noisy. Specifically, $q_{\theta_h}^t$ represents the maximum outcome reward $r_\theta$ given $\mathbf{y}_{<t}$, rather than the expectation, thus overestimating the Q value; 
For $q_{\theta_s}^t$, given the limited capability of the policy model in practice, it can be challenging to sample correct solutions for difficult instructions, suffering from false negative noises and thus underestimating Q.

\section{Implicit PRMs For Free Through Reward Parameterization}
\label{sec:theorem}
\looseness=-1
In this section, we show that an ORM can directly represent an expectation of the outcome reward by itself by simple \textbf{reward parameterization}. In other words, a PRM can be inherently derived from the same ORM without any dedicated training, 
offering better performance than MCTS-based approaches with substantially lower overhead.

\paragraph{Reward parameterization in existing work}
Current literature typically parameterize rewards by either
(1) the linear transformation of hidden states, with the reward model being a sequence classifier \citep{Ouyang2022TrainingLM, Touvron2023Llama2O, starling2023, Cui2023ULTRAFEEDBACKBL} 
or (2) generative logits, with reward models being an auto-regressive LM and trained to predict the label of partial or complete responses as ``good'' or ``bad'' tokens, and sometimes a third ``neutral'' \citep{Zhang2024GenerativeVR, Mahan2024GenerativeRM, Lightman2023LetsVS, Wang2023MathShepherdVA, Luo2024ImproveMR}.

Unfortunately, under either of the two parameterizations, PRMs would require expensive step labels to train. To address this issue, \textbf{we propose to train an ORM with implicit reward modeling, which will automatically enable a PRM regardless of the loss functions.} Next, we illustrate this in detail:

\begin{proposition}\label{proposition_1}
(Proof in Appendix \ref{sec:appendix_proof}) Consider an ORM where the reward is parameterized by the log-likelihood ratio of two causal LMs, i.e. $r_\theta(\mathbf{y}):= \beta \log \frac{\pi_\theta(\mathbf{y})}{\pi_\text{ref}(\mathbf{y})}$. 
Define $q_\theta^t(\mathbf{y}_{<t}, y_t):= \sum_{i=1}^{t} \beta \log \frac{\pi_\theta(y_{i}|\mathbf{y}_{<i})}{\pi_\text{ref}(y_{i}|\mathbf{y}_{<i})}$. 
$q_\theta^t$ is the exponential average of $r_\theta$ at step $t$.
\begin{equation}
q_\theta^t(\mathbf{y}_{<t}, y_t) = \beta \log \E_{\pi_\text{ref}(\mathbf{y}|\mathbf{y}_{\leq t})} e^{\frac{1}{\beta}r_\theta(\mathbf{y})}
\end{equation}
Hence, \textbf{$q_\theta^t$ represents an exact expectation of outcome reward $r_\theta$ at step $t$, i.e., the Q value}.
\end{proposition}

Proposition \ref{proposition_1} indicates that when modeling $r_\theta(\mathbf{y}):= \beta \log \frac{\pi_\theta(\mathbf{y})}{\pi_\text{ref}(\mathbf{y})}$ to train an ORM with the standard pipeline, where $\beta$ is a hyperparameter, $\theta$ can implicitly learn a Q function. Hence, process reward $r_\theta^t$ can be obtained by:
\begin{equation}
\label{eq:process_reward}
r_\theta^t := q_\theta^t - q_\theta^{t-1} = \sum_{i=t-1}^{t} \beta \log \frac{\pi_\theta(y_{i}|\mathbf{y}_{<i})}{\pi_\text{ref}(y_{i}|\mathbf{y}_{<i})}
\end{equation}
Notably, this conclusion still holds when $y_t$ represents the $t$-th token rather than step $t$.
\textbf{This gives us an inspiring hint: we can indeed obtain PRMs, or more fine-grained token-level RMs, simply by collecting response-level data and training an ORM, without any burden of annotating step labels.}
The proposition is agnostic to specific choices of the training objective of ORMs. It can be instantiated with different objectives as vanilla ORM training, with the only difference being substituting the $r_\theta \left( \mathbf{y} \right)$ with $\beta \log \frac{\pi_\theta(\mathbf{y})}{\pi_\text{ref}(\mathbf{y})}$.
Particularly, many existing preference learning algorithms have already met our assumption \citep{Rafailov2023DirectPO, Azar2023IPO, Ethayarajh2024KTOMA, Chen2024NoiseCA, Wu2024SelfPlayPO}.

Besides making PRM training more accessible, our implicit process reward can be more accurate than those derived from $q_{\theta_{s}}^t$ and $q_{\theta_{h}}^t$ in Eq. \ref{eq:mcts_q} \citep{Wang2023MathShepherdVA}, as indicated by the following proposition:

\begin{proposition}\label{proposition_2}
The performance of $q_\theta^t$ is guaranteed by the following conditions:
\textbf{$q_\theta^t$ is bounded by $q_{\theta_{s}}^t$ and $q_{\theta_{h}}^t$, and can reach these bounds with specific values of $\beta$}. That is,
\begin{equation}
q_{\theta_{s}}^t = \E_{\pi_\text{ref}(\mathbf{y}|\mathbf{y}_{<t}) } r_\theta(\mathbf{y}) \leq q_\theta^t(\mathbf{y}_{<t}, y_t) \leq \max_{\mathbf{y}|\mathbf{y}_{<t}} r_\theta(\mathbf{y}) = q_{\theta_{h}}^t
\end{equation}
holds. The left-hand equality is attained as $\beta \to \infty$ and the right-hand one is attained as $\beta \to 0$.
\end{proposition}


\looseness=-1
Proposition \ref{proposition_2} demonstrates that $q_\theta^t$ ranges between the soft-estimated and hard-estimated Q values annotated by MCTS-based approaches. The above bounds suggest that our approach has better accuracy and robustness to noises than MCTS-based approaches. 
Specifically, as discussed in \S \ref{sec:bg_dilemma}, $q_{\theta_h}^t$ overestimates the Q value while $q_{\theta_s}^t$ underestimates Q due to false negative noises.
Since $q_\theta^t$ lies between $q_{\theta_h}^t$ and $q_{\theta_s}^t$, it could potentially mitigate both issues and estimate the Q value more accurately.
Concurrent work defines our $q_\theta^t$ as an entropy regularized process reward and has empirically shown its superiority over $q_{\theta_s}^t$ and $q_{\theta_h}^t$ on best-of-N sampling \citep{Zhang2024ERPRM}.

\paragraph{Connection to \cite{Rafailov2024FromT}}
An intuition similar to Proposition \ref{proposition_1} has been brought up by \cite{Rafailov2024FromT}, which demonstrates that DPO enables models to learn the Q function implicitly, but our insights subsume their conclusion since this property is not limited to the DPO algorithm. 
For example, given response-level label $l$, we can further generalize to cross-entropy (CE) loss to handle practical scenarios with unpaired and imbalanced data:
\begin{equation}
\label{eq:ce}
    \mathcal{L}_{CE} = l \cdot \log \sigma \left( \beta \log \frac{\pi_\theta(\mathbf{y})}{\pi_\text{ref}(\mathbf{y})} \right) + (1-l) \cdot \log\left[ 1 - \sigma \left( \beta \log \frac{\pi_\theta(\mathbf{y})}{\pi_\text{ref}(\mathbf{y})} \right) \right]
\end{equation}

\paragraph{Reference Model} One difference between our modeling of rewards and previous ones is the incorporation of a reference model $\pi_\text{ref}$. 
We acknowledge that this comes at an inference cost: to calculate the reward, both the policy and reference model are served, which doubles the inference cost than vanilla PRM.
However, it is prevalent in existing preference learning algorithms and works as the KL constraint to prevent the policy model $\pi_\theta$ deviating too far from its starting checkpoint.
Moreover, it is less a problem in practice, as we will show in \S \ref{sec:ref_practical_cost} that a large proportion of the inference overhead in best-of-N sampling comes from the generation model, especially when the generation model is much larger than the reward model.
Further, we also show in \S \ref{sec:ref_remove_ref} that when the implicit PRM is built from a strong model that has undergone preference learning, such as Llama-3.1-Instruct,
excluding $\pi_{\text{ref}}$ leads to little or no accuracy drop. This makes our approach appealing in practice since it can achieve better accuracy than existing PRMs with exactly the same inference overhead, but substantially lower development overhead.

\section{Experiments}
\looseness=-1
\label{sec:main_exp}

\subsection{Setup}
\paragraph{Evaluation}
Following standard practice \citep{Lightman2023LetsVS}, we evaluate PRMs with best-of-N (BoN) on MATH-500 \citep{Hendrycks2021MeasuringMP}.
To study the generalizability of the PRMs, 
we test each PRM using three generation models with different levels of capabilities:
Mistral-Instruct-v0.3 \citep{Jiang2023Mistral7}, Llama-3.1-8B-Instruct, and Llama-3.1-70B-Instruct \citep{llama3modelcard}.
For each completion, we apply PRMs to score each step and pick the lowest step reward as the score for overall responses.
We also compare the development overhead of the models in terms of FLOPs, including those required in both the automatic data collection and PRM training. 

\paragraph{Training dataset}
Unless stated otherwise, we adopt the following training setup throughout all experiments:
We use math instructions from UltraInteract \citep{Yuan2024AdvancingLR} and sample eight rollouts per instruction using Llama-3.1-8B-Instruct, and then assess rollout correctness with ground truths. 
We train PRMs based on Llama-3.1-8B-Instruct with $\beta=0.05$, which is empirically determined.

\paragraph{Implicit PRM instantiation}
As demonstrated in \S\ref{sec:theorem}, our approach can be instantiated with any reward modeling objective with the reward parameterized as 
$r_\theta := \beta \log \frac{\pi_\theta(\mathbf{y})}{\pi_\text{ref}(\mathbf{y})}$.
We explore various objectives that meet the requirements, including DPO \citep{Rafailov2023DirectPO}, KTO \citep{Ethayarajh2024KTOMA}, NCA \citep{Chen2024NoiseCA}, and the cross-entropy (CE) loss. Please refer to Eq. \ref{eq:ce} for the implementation of CE loss.
For DPO and NCA, we pair each correct rollout with an incorrect counterpart and train our RM on these response-level pairs, while for KTO and CE loss, we directly train on the unpaired and imbalanced rollouts, which is more general in practical scenarios.
We also implement two data balanced setup for CE to analyze the impact of pairwise data, i.e. balancing the positive and negative responses simply for the entire dataset, or more strictly for the each each instruction. We denote the two setups as Dataset-wise Balanced and Instruction-wise Balanceed.

\paragraph{Baselines}
\looseness=-1
Our baselines include our implementation of existing methods and off-the-shelf open models. 
We reimplement Math-Shepherd \citep{Wang2023MathShepherdVA} and AutoPSV \citep{Lu2024AutoCVER} for fair comparisons, representative algorithms in their categories.
Math-Shepherd annotates step labels using MCTS estimations as illustrated in \S \ref{sec:bg_dilemma}.
AutoPSV annotates steps with a two-stage strategy. 
It firsts trains an outcome supervision verifier (OSV) that predicts Q value for each step, then use the OSV to annotate step labels.
A PRM is obtained by continual training on the OSV with process labels.
We also compare to six off-the-shelf ORMs and PRMs, namely EurusRM-7B \citep{Yuan2024AdvancingLR}, SkyworkRM-Llama3.1-8B \citep{liu2024skywork}, ArmoRM-Llama3-8B \citep{ArmoRM}, Math-Shepherd-7B (the offical release of \cite{Wang2023MathShepherdVA}), RLHFlow-8B-Mistral-Data\footnote{\url{https://huggingface.co/RLHFlow/Llama3.1-8B-PRM-Mistral-Data}}, and RLHFlow-8B-DS-Data\footnote{\url{https://huggingface.co/RLHFlow/Llama3.1-8B-PRM-DeepSeek-Data}}.
We note that these off-the-shelf baselines are trained on different instructions and responses, while our two reimplementations are trained on the same data as our implicit PRM.

\subsection{Results}
\begin{table}[]
\centering
\caption{Different reward models' best-of-N sampling performance on MATH test set with three different generation models. When completing instructions with a temperature of 0.5, the three generation models' accuracies  are 9.6\%, 44.6\%, and 63.2\% respectively.}
\label{tab:main-exp}
\resizebox{\textwidth}{!}{
\begin{tabular}{@{}l|l|ccccccccc|c@{}}
\toprule
\multirow{2}{*}{\bf Type}                                                             & \multirow{2}{*}{\bf Reward Model}        & \multicolumn{3}{c|}{\makecell{\bf Mistral-7B-Inst-v0.2\\ \bf Pass@1: 9.6}} & \multicolumn{3}{c|}{\makecell{\bf Llama-3.1-8B-Inst\\ \bf Pass@1: 44.6}}   & \multicolumn{3}{c|}{\makecell{\bf Llama-3.1-70B-Inst\\ \bf Pass@1: 63.2}} & \multirow{2}{*}{\bf Avg.} \\ \cmidrule(lr){3-11}
&                               & @4    & @16  & \multicolumn{1}{c|}{@64}  & @4   & @16  & \multicolumn{1}{c|}{@64}  & @4          & @16         & @64         &                       \\ \midrule
\multicolumn{12}{c}{\it Open-Source Reward Models} \\ \midrule                                                                                  
\multirow{3}{*}{\bf ORM}                                                              & \href{https://huggingface.co/openbmb/Eurus-RM-7b}{EurusRM-7B}                       & 17.2  & 21.0 & \multicolumn{1}{c|}{20.4} & 49.6 & 51.6 & \multicolumn{1}{c|}{51.8} & 69.0        & 69.6        & 72.2        & 46.9                  \\
& \href{https://huggingface.co/Skywork/Skywork-Reward-Llama-3.1-8B-v0.2}{SkyworkRM-Llama3.1-8B}                    & 16.0  & 19.6 & \multicolumn{1}{c|}{23.4} & 49.0 & 50.4 & \multicolumn{1}{c|}{48.2} & 70.4        & 72.6        & 72.0        & 46.8                  \\
& \href{https://huggingface.co/RLHFlow/ArmoRM-Llama3-8B-v0.1}{ArmoRM-Llama3-8B}                        & 16.6  & 21.0 & \multicolumn{1}{c|}{23.2} & 47.8 & 48.6 & \multicolumn{1}{c|}{49.4} & 70.6        & 70.8        & 71.0        & 46.6                  \\ \midrule
\multirow{3}{*}{\bf PRM}                                                              & \href{https://huggingface.co/peiyi9979/math-shepherd-mistral-7b-prm}{Math-Shepherd-7B}  & 16.0  & 21.0 & \multicolumn{1}{c|}{20.4} & 50.0 & 52.4 & \multicolumn{1}{c|}{52.8} & 66.4        & 65.8        & 65.6        & 45.6                  \\
& \href{https://huggingface.co/RLHFlow/Llama3.1-8B-PRM-Mistral-Data}{RLHFlow-8B-Mistral-Data}  & {\bf 19.4}  & {\bf 25.2} & \multicolumn{1}{c|}{\bf 30.2} & 51.8 & 52.0 & \multicolumn{1}{c|}{50.6} & 70.8        & 71.0        & 71.2        & 49.1                  \\
& \href{https://huggingface.co/RLHFlow/Llama3.1-8B-PRM-Deepseek-Data}{RLHFlow-8B-DS-Data} & 17.2  & 23.0 & \multicolumn{1}{c|}{25.2} & {\bf 54.4} & 54.2 & \multicolumn{1}{c|}{55.8} & 68.6        & 70.4        & {\bf 73.0}        & 49.1                  \\ \midrule
\multicolumn{12}{c}{\it Our Implementations} \\ \midrule    
\multirow{2}{*}{\begin{tabular}[c]{@{}l@{}} \bf Baselines\end{tabular}} & {\bf Math-Shepherd}                 & 17.6  & 24.4 & \multicolumn{1}{c|}{26.8} & 50.0 & 51.4 & \multicolumn{1}{c|}{52.8} & 68.6        & 69.4        & 68.8        & 47.8                  \\
& {\bf AutoPSV}                       & 16.6  & 20.6 & \multicolumn{1}{c|}{22.2} & 52.2 & 51.4 & \multicolumn{1}{c|}{52.2} & 68.4        & 65.4        & 62.4        & 45.7                  \\ \midrule
\multirow{6}{*}{\begin{tabular}[c]{@{}l@{}} \bf Implicit PRM\end{tabular}} & \multicolumn{1}{l|}{\textbf{DPO}}                                                                                                     & 18.6                           & 24.4                           & \multicolumn{1}{c|}{28.8}                           & 54.0                           & {\bf 55.4}                          & \multicolumn{1}{c|}{\bf 57.0}                          & 71.8                           & 71.2                           & \multicolumn{1}{c|}{72.2}                          & \textbf{50.4}                            \\
\multicolumn{1}{l|}{}                                                 & \multicolumn{1}{l|}{\textbf{KTO}}                                                                                                     & 15.6                           & 18.4                           & \multicolumn{1}{c|}{18.6}                           & 49.6                           & 51.8                          & \multicolumn{1}{c|}{50.8}                          & {\bf 72.6}                           & 67.0                           & \multicolumn{1}{c|}{67.2}                          & 45.7                                     \\
\multicolumn{1}{l|}{}                                                 & \multicolumn{1}{l|}{\textbf{NCA}}                                                                                                     & 18.6                           & 23.8                           & \multicolumn{1}{c|}{28.0}                           & 52.4                           & 53.4                          & \multicolumn{1}{c|}{55.2}                          & 69.0                           & {\bf 73.0}                           & \multicolumn{1}{c|}{71.6}                          & 49.4                                     \\
\multicolumn{1}{l|}{}                                                 & \multicolumn{1}{l|}{\textbf{CE}}                                                                                                      & 18.8                           & 24.0                           & \multicolumn{1}{c|}{28.0}                           & 52.6                           & 54.4                          & \multicolumn{1}{c|}{53.0}                          & 70.6                           & 67.0                           & \multicolumn{1}{c|}{67.2}                          & 48.4                                     \\
\multicolumn{1}{l|}{}                                                 & \multicolumn{1}{l|}{\textbf{CE (Dataset-wise Balanced)}}                                                                              & 18.0                           & 23.6                           & \multicolumn{1}{c|}{27.0}                           & 52.6                           & 54.2                          & \multicolumn{1}{c|}{52.6}                          & 68.6                           & 66.8                           & \multicolumn{1}{c|}{67.0}                          & 47.8                                     \\ 
\multicolumn{1}{l|}{}                                                 & \multicolumn{1}{l|}{\textbf{CE (Inst.-wise Balanced)}}                                                                          & 17.6                           & 22.6                           & \multicolumn{1}{c|}{26.2}                           & 52.6                           & 55.2                          & \multicolumn{1}{c|}{54.6}                          & 69.4                           & 71.2                           & \multicolumn{1}{c|}{72.0}                          & 49.0                                     \\ \bottomrule
\end{tabular}
}
\end{table}

\paragraph{Various implicit reward modeling objectives outperform baselines}
\looseness=-1
According to BoN results shown in Table \ref{tab:main-exp}, all four variants of our implicit PRMs
consistently improve the accuracies of the three different generation models. 
Among them, DPO achieves an averaged accuracy of 50.4, performing better in general, closely followed by NCA with an averaged accuracy of 49.4.
CE presents strong performance, despite that it is trained on unpaired and imbalanced data. 
Specifically, with an averaged accuracy of 48.4, it beats our implemented Math-Shepherd and AutoPSV by 0.6 and 2.7 respectively, and outperforms other open-source reward models except RLHFlow-8B-Mistral-Data and RLHFlow-8B-DS-Data, both of which achieves 49.1.
This indicates the potential in empowering real-world applications where pairwise data is hard to collect.
Nevertheless, according to CE versus CE (Inst.-wise Balanced), it is still beneficial to have balanced positive and negative responses for each instruction in the training dataset, which aligns with conventional understandings on CE as a classification loss. However, comparing CE (Dataset-wise Balanced) to CE, simply balancing the entire dataset by randomly filtering examples of the class with more data can be detrimental.

\begin{figure}[tbh]
    \centering
    \includegraphics[width=.85\linewidth]{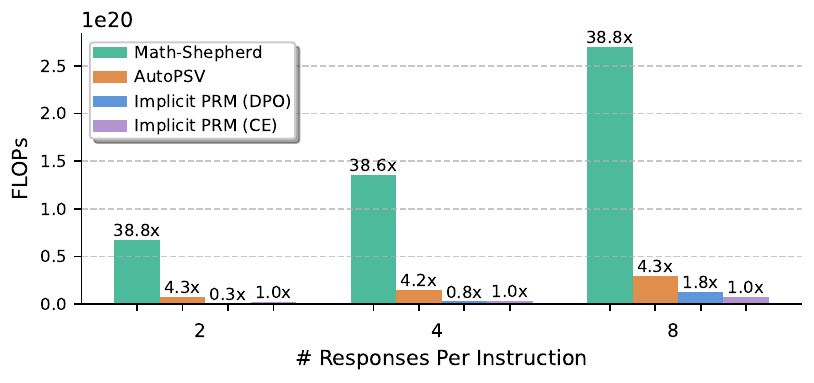}
    \vspace{-10pt}
    \caption{
    Overhead of developing different PRMs, in terms of FLOPs during data collection and training.
    The X axis indicates the number of responses per instruction which determines the scale of training data, and the Y axis is the number of FLOPs. Our implicit PRM always consumes the least FLOPs compared to baselines, with CE being 38.6$\times$ to 38.8$\times$ more efficient than Math-Shepherd across different dataset scales.
    }
    \label{fig:tokens}
\end{figure}
\paragraph{Our Implicit PRMs reduce the overhead of data collection and training by $38.8\times$
}
\label{sec:efficiency}

As shown in Figure \ref{fig:tokens}, \textbf{with three different training data scales. Math-Shepherd generally costs 38.8x more FLOPs than the implicit PRM (CE).} Compared to implicit PRM (DPO), the number becomes 146.5x, 49.9x, and 21.3x under different number of responses per instruction respectively.

We plot the scaling trends of the average performance of each method with corresponding number of tokens consumed in Figure \ref{fig:token-acc}, from which we can clearly see that our implicit PRMs achieve better performance with much less data collection and training overhead.

\section{Analysis}
\begin{wrapfigure}{r}{0.45\textwidth}
    \centering
    \vspace{-7pt}
    \includegraphics[width=0.45\textwidth]{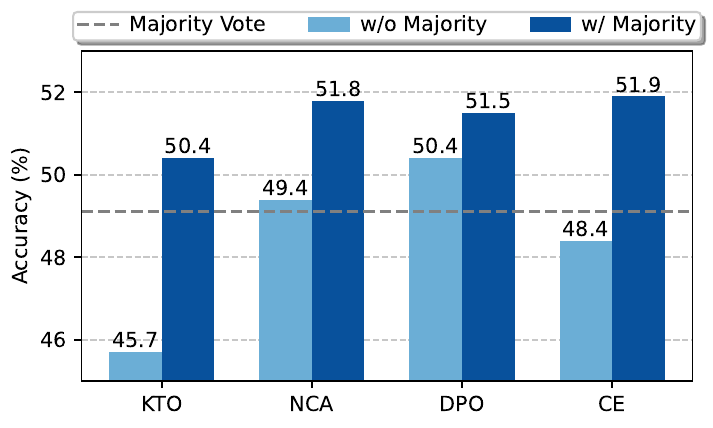}
    \vspace{-23pt}
    \caption{\looseness=-1
    Results with majority voting. We present the averaged best-of-N accuracy across three testsets.}
    \label{fig:label}
    \label{fig:sc}
    \vspace{-18pt}
\end{wrapfigure}

\subsection{Incorporating Majority Voting}
\looseness=-1
Our implicit PRMs can be integrated with majority voting to improve the performance even further.
Previously, we apply our implicit PRMs to score each response and pick the response with highest individual score as the final answer. However, when incorporating with majority voting, the scores of responses that lead to the same answer will be aggregated and the answer with the highest aggregated score will be selected as the final answer.
We present the results averaged over different numbers of candidate solutions per problems across all three generated models in Figure~\ref{fig:sc}.

We observe that our implicit PRM can successfully adjust voting distributions, and achieves better results than using the implicit PRM or majority voting separately.
Particularly, KTO and CE variants gain the most from the integration, both of which fail to surpass majority voting alone but outperforms it through weighted best-of-N. 
It is also noteworthy that CE loss become the most effective when augmented with majority voting, once again demonstrating its potential.

\subsection{Scaling Up Instructions and Responses can Improve Implicit PRMs}
\label{sec:scaling}
\begin{figure}[tbh]
    \centering
    \begin{subfigure}{\textwidth}
        \includegraphics[width=\linewidth]{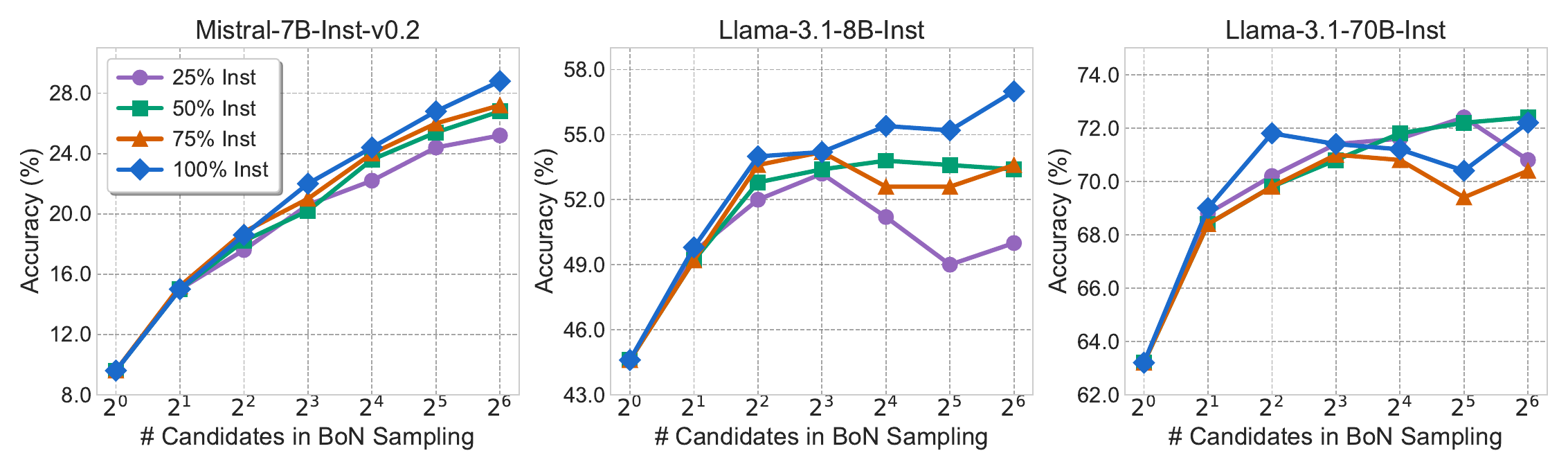}
        \caption{Implicit PRM (DPO).}
    \end{subfigure}

    \begin{subfigure}{\textwidth}
        \includegraphics[width=\linewidth]{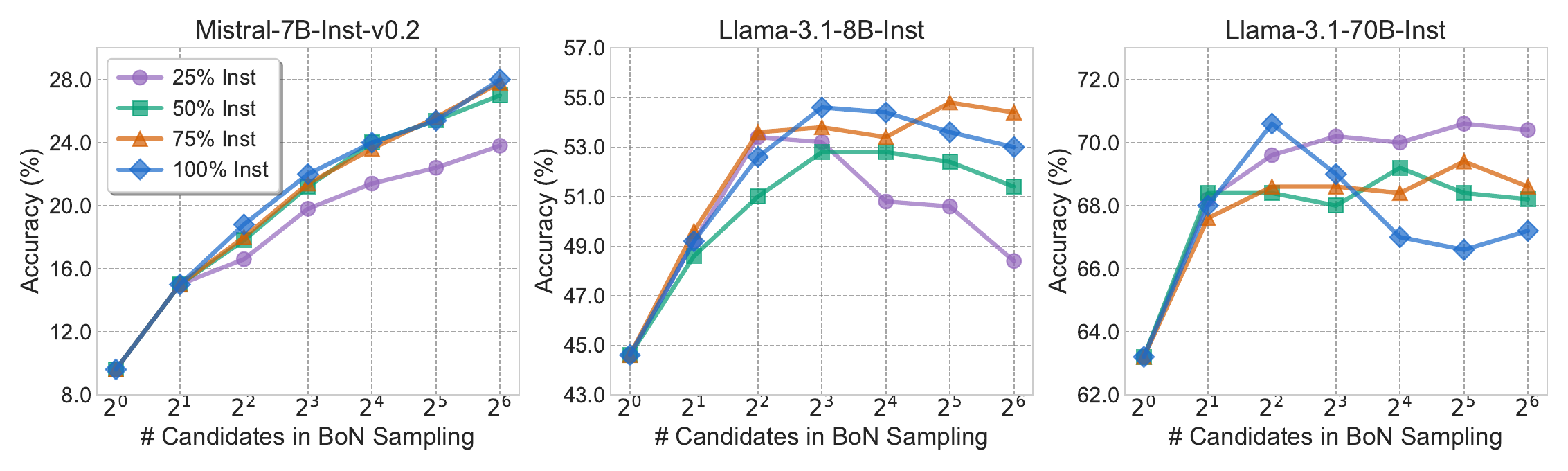}
        \caption{Implicit PRM (CE).}
    \end{subfigure}
    
    \caption{Scaling instruction numbers. Our implicit PRM's performance on Mistral-7B-Instruct-v0.2 and Llama-3.1-8B-Instruct scales well with the number of instructions, despite the trend is more complex on Llama-3.1-70B-Instruct.}
    \label{fig:inst_scaling}
\end{figure}
\begin{figure}[tbh]
    \centering

    \begin{subfigure}{\textwidth}
        \includegraphics[width=\linewidth]{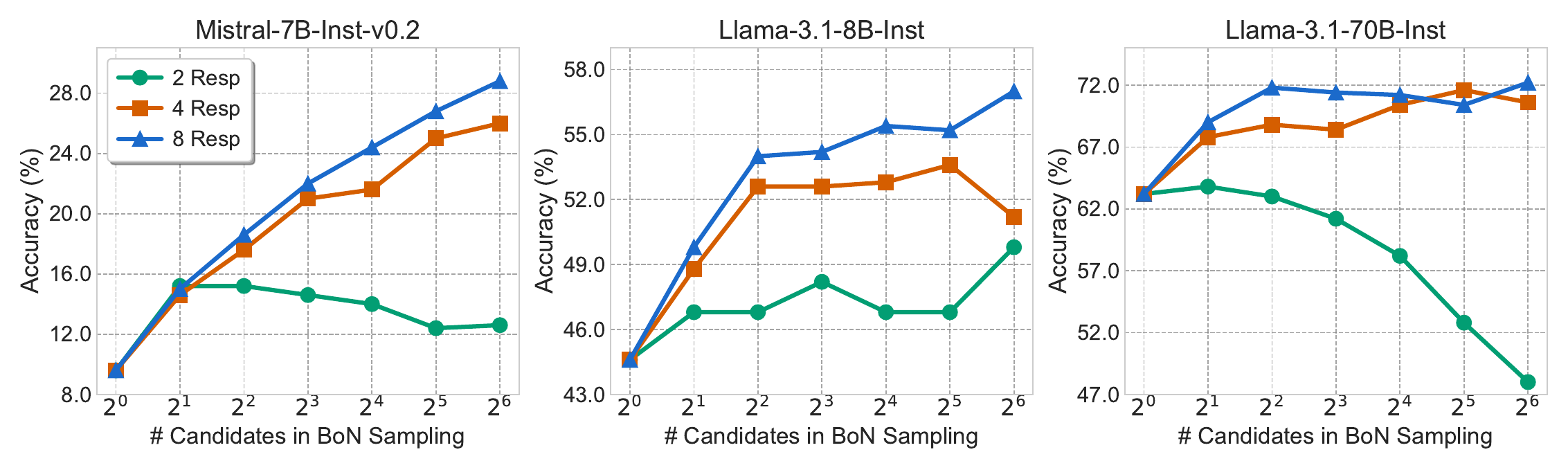}
        \caption{Implicit PRM (DPO).}
        \label{fig:resp_scaling_dpo}
    \end{subfigure}
    
    \begin{subfigure}{\textwidth}
        \includegraphics[width=\linewidth]{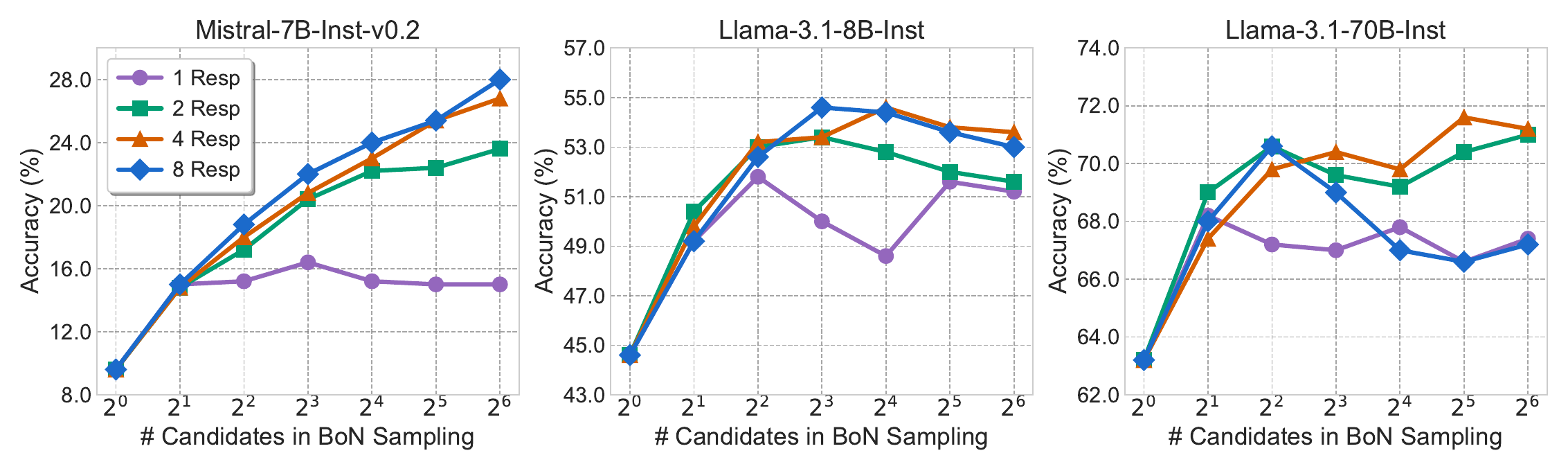}
        \caption{Implicit PRM (CE). Note that one repsonse per instruction is the extreme case of the unpaired setup.}
        \label{fig:resp_scaling_ce}
    \end{subfigure}
    
    \caption{\looseness=-1
    Scaling responses number for each instruction. Our implicit PRM generally benefits from scaling up the number of responese for each instruction. Particularly, DPO is under-trained with two responses per instruction. This can be partly attributed to the insufficient amount of instructions: two responses may not constitute a pair to train our DPO variant, and thus many instructions can not be used in training.
    In contrast, CE generally performs better with insufficient data and can always improve different generation model, even when it is trained with one response per instruction with pairs.}
    \label{fig:resp_scaling}
\end{figure}

\paragraph{Setup}
We conduct scaling analysis with DPO and CE on both instructions and responses of the training dataset. For instruction scaling, we randomly down sample 25\%, 50\%, and 75\% instructions to train our implicit PRM. 
For response scaling, since DPO can only train on paired responses, we train models with 2, 4, and 8 rollouts respectively; while for CE, we also implement training with \textit{only one rollout per instruction}, the extreme case of unpaired setup. 

\paragraph{Results}
We present results in Figure \ref{fig:inst_scaling} and Figure \ref{fig:resp_scaling} respectively. Takeaways are summarized as follows:
(1)
\textbf{Scaling instructions and responses consistently improve the performance of our implicit PRM.} The trend is particularly clear on Mistral-7B-Inst-v0.2 and Llama-3.1-8B-Inst, but there are also a few outliers on Llama-3.1-70B-Inst.
(2)
\textbf{Compared to instructions, scaling up responses seems to be more influential on implicit PRMs}, as reflected by the larger performance variations between the minimum and maximum data setups.
Taking a closer look at the response scaling, 
(3)
\textbf{DPO requires more data to obtain a descent performance than CE.}
From Figure \ref{fig:resp_scaling_dpo}, DPO is under-trained with two responses per instruction, which can be partly attributed to the insufficient amount of instructions: two responses may not constitute a pair to train our DPO variant, and thus many instructions can not be used in training.
In contrast, CE generally performs better with insufficient data and can always improve different generation model, even when it is trained with one response per instruction with pairs, the extreme case of the unpaired setup.
This presents a huge advantage in real-worl data scarcity scenarios.

\subsection{Are There Any Other Factors can Improve Implicit PRM Performance?}
We consider potential factors that may influence the performance of implicit PRMs, as listed below:

\paragraph{Task-irrelevant Instructions}
We previously only consider math instructions. We now examine if increasing instructions diversity, even if the instructions are irrelevant to downstream tasks, can benefit implicit PRMs. 
To this end, we incorporate general instructions from UltraFeedback \citep{Cui2023ULTRAFEEDBACKBL} and coding instructions from UltraInteract \citep{Yuan2024AdvancingLR} into our training dataset. We directly use responses from the original datasets, but for UltraFeedback we only randomly select one pair for each instruction, instead of using all the pairs.

\paragraph{Response Diversity}
We first conduct a deduplication on our preference dataset based on 8-gram overlap, aiming to verify if repeated responses hurt model performance.
We then randomly replace four rollouts per instruction in the original training dataset with another four rollouts generated by Llama-3.1-8B-Base model.

\looseness=-1
\paragraph{Training on Step Labels}
Our implicit PRMs do not  require step labels for training. 
However, we are interested in exploring whether augmenting them with step labels can further improve their performance. 
Based on the definition of process labels, we adjust the implicit reward of a step by increasing it for positive labels and decreasing it for negative ones. 
We use the labels obtained from our implemented Math-Shepherd, which has been demonstrated to be a strong implementation with step labels of high-quality (\S \ref{sec:main_exp}).
We adapt KTO to a step-level version for optimization. Therefore, considering a $n$-step response with step labels $\{l^1, l^2, \dots, l^n\}$, we conduct a \textit{second stage} training on our current implicit PRM to explicitly optimize the implicit reward:~$\mathcal{L_\theta}  = - \frac{1}{n} \sum_{t=1}^{n} \log \left( \sigma \left(l^t \cdot  \left | r^t_\theta \right |  \right )\right )$.
\begin{table}[]
\centering
\caption{Factors that may affect PRM performance. To our surprise, none of them consistently improve our implicit PRM.}
\label{tab:factors}
\vspace{-7pt}
\resizebox{0.9\textwidth}{!}{
\begin{tabular}{@{}l|ccc|ccc|ccc|c@{}}
\toprule
\multirow{2}{*}{Setup} & \multicolumn{3}{c|}{Mistral-7B-Inst-v0.2} & \multicolumn{3}{c|}{Llama-3.1-8B-Inst} & \multicolumn{3}{c|}{Llama-3.1-70B-Inst} & \multirow{2}{*}{Avg.} \\ \cmidrule(lr){2-10}
                       & @4            & @16           & @64           & @4           & @16          & @64          & @4            & @16          & @64          &                       \\ \midrule
Implicit PRM           & 18.6          & 24.4          & 28.8          & 54.0         & 55.4         & 57.0         & 71.8          & 71.2         & 72.2         & 49.3                  \\ \midrule
+ UltraFeedback        & 19.4          & 24.4          & 29.0          & 53.8         & 55.0         & 55.8         & 71.6          & 70.6         & 72.2         & 49.2                  \\
+ UltraInteract (Code) & 19.2          & 24.6          & 28.0          & 54.6         & 54.0         & 56.8         & 71.4          & 70.8         & 70.0         & 49.2                  \\ \midrule
+ Dedup.               & 18.2          & 22.8          & 26.8          & 52.0         & 53.2         & 51.6         & 69.8          & 69.4         & 70.4         & 47.6                  \\
+ Base Resp.           & 17.8          & 23.2          & 27.6          & 54.0         & 55.0         & 54.8         & 71.4          & 72.4         & 73.2         & 48.7                  \\
+ Step Label           & 18.8          & 25.4          & 28.8          & 53.8         & 54.8         & 54.6         & 70.8          & 71.2         & 73.0         & 49.2                  \\ \bottomrule
\end{tabular}
}
\end{table}

\paragraph{Results}
We present results on implicit PRM (DPO) in Table \ref{tab:factors}. In general, \textbf{none of these factors brings consistent gains.}
(1)
Both adding UltraFeedback and UltraInteract (code) instructions hurt the performance, with the former suffers more severely. This implies that training instructions deviating from the downstream task could undermine the performance of implicit PRMs.
(2)
Regarding response diversity, we observe that the performance of deduplicating responses hurts the performance and is close to implicit PRMs trained on similar amount of data. This indicates that repeated responses function similarly as others and are still beneficial before model performance saturates. 
Replacing part of original rollouts with those generated by the base model also fails to improve performance.
(3)
Conducting step-level KTO with extra process labels does not bring gains, reinforcing our claim that we can already train a strong PRM without process label. 
However, one should be cautious about concluding that stepwise labels are generally not helpful due to two factors in our experiments:
Firstly, despite our efforts that lead to improved step annotation quality compared to previous work, the MCTS-based approach inevitably introduces noises in the data annotation process, as we discussed in \S \ref{sec:bg_dilemma}; 
Secondly, our choice of algorithm may not be optimal. It is possible that more advanced PRM data annotation methods and training algorithms can finally integrate information from (noisy) stepwise labels into implicit PRM.

\subsection{PRM Ability Does Not Translate into Policy Performance}
\begin{wraptable}{r}{0.28\textwidth}
\centering
\vspace{-10pt}
\caption{
\looseness=-1
Implicit PRMs' performance on MATH500 when used to solve the problems directly.}
\vspace{-7pt}
\label{tab:policy}
\resizebox{0.28\textwidth}{!}{
\begin{tabular}{@{}l|c@{}}
\toprule
Model             & Accuracy \\ \midrule
Llama-3.1-8B-Inst & 45.2     \\
~+ DPO             & 25.8     \\
~+ KTO             & 46.6     \\
~+ NCA             & 35.6     \\
~+ CE              & 28.6     \\ \bottomrule
\end{tabular}
}
\vspace{-15pt}
\end{wraptable}
Implicit PRM is trained in an auto-regressive manner, sometimes directly using preference learning algorithms, which are primarily used to improve policy models.
Therefore, it reserves the nature as a causal LM and can still serve as a policy model to solve downstream problems directly.
In this section, we test on MATH500 \citep{Hendrycks2021MeasuringMP, Lightman2023LetsVS} to analyze the correlation between their PRM ability and performance as a policy model. 

According to Table \ref{tab:policy}, only trainiing with KTO leads to an improvement on MATH500, compared to Llama-3.1-8B-Instruct. Interestingly, based on Table \ref{tab:main-exp}, KTO performs the worst as an implicit PRM. In contrast, DPO and CE, the two algorithms that perform the best in without majority voting and with majority voting setups, respectively, achieve the lowest accuracies. This indicates that PRM ability does not improve as the policy model improves, and there can even be an unexpected trade-off between the both abilities.

\subsection{Can We Reduce the Inference Overhead of the Reference Model?}
\label{sec:ref}
One concern on our approach is the need of an additional reference model at inference. 
However, we show that the the reference model does not double overall inference overhead in practice, especially when the generation model is much larger than the reward model (\S \ref{sec:ref_practical_cost}). Next, in \S \ref{sec:ref_remove_ref}, we show that the reference model can be removed at inference in certain cases.

\subsubsection{The Reference Model Does not Double Overall Inference Overhead}
\label{sec:ref_practical_cost}

\paragraph{Setup}
\looseness=-1
We calculate the time costs of best-of-N sampling on MATH500 in practice. The entire process includes (1) generating multiple candidate solutions to the instruction using the generation model, and (2) scoring each candidate using a PRM. We use vLLM \citep{kwon2023efficient} to implement the former and Huggingface Accelerate \citep{accelerate} for the latter.

\begin{wraptable}{r}{0.7\textwidth}
\centering
\vspace{-11pt}
\caption{GPU time costs during best-of-N sampling relative to the cost of generation model (\%). The overall inference overhead of baselines on three test sets are 66.6\%, 70.8\%, and 90.9\% of that of our implicit PRM, respectively. Namely, the reference model does not double the inference cost in practice, and the extra inference overhead becomes more marginal as the generation model gets larger.}
\label{tab:ref_time_cost}
\vspace{-8pt}
\resizebox{0.7\textwidth}{!}{
\begin{tabular}{@{}l|l|c|c|c@{}}
\toprule
Source of Cost                & Method       & Mistral-7B-Inst-v0.2 & Llama-3.1-8B-Inst & Llama-3.1-70B-Inst \\ \midrule
Generation Model              & -            & 100.0               & 100.0            & 100.0            \\ \midrule
\multirow{2}{*}{Reward Model} & Baselines    & 33.5               & 29.4            & 9.1             \\
                              & Implicit PRM & 201.6              & 141.7            & 22.2            \\ \midrule
\multirow{2}{*}{Total}        & Baselines    & 200.9              & 171.1            & 111.1            \\
                              & Implicit PRM & 301.6              & 241.7           & 122.2            \\ \bottomrule
\end{tabular}
}
\vspace{-15pt}
\end{wraptable}
\paragraph{Results}
We present the GPU time costs on A100 80G relative to that of the generation model in Table \ref{tab:ref_time_cost}. We find that the inference overhead from generation model takes a large proportion of the total overhead, especially when the generation model is much larger than the reward model.
Therefore, the reference model in our implicit PRM does not double the overall inference cost in practice: The overall inference overhead of baselines on three test sets are 66.6\%, 70.8\%, and 90.9\% of that of ours, respectively. 
It is noteworthy that the extra overhead introduced by the reference model becomes more marginal as the generation model b larger, and is almost negligible when Llama-3.1-70B-Instruct serves as the generation model.

\subsubsection{The Reference Model Can be Removed at Inference in Certain Cases}
\label{sec:ref_remove_ref}

We note that our proposition still holds under a uniformly distributed reference model, i.e. $\log \pi_\text{ref} = constant$. In best-of-N sampling, only relative scores between steps or responses matter, where the constant $\log \pi_\text{ref}$ can be canceled out, equivalent to exclude the reference model in reward parametrization.
Therefore, we derive a more efficient implementation of our proposition by removing the reference model.
We examine its effectiveness and explore if we can simply our method to reduce the inference overhead in practice.

\begin{table}[t]
\centering
\caption{Ablating reference model in both training and inference. Neither consistently hurts our implicit PRM. More surprisingly, the reference model, Llama-3.1-8B-Instruct, already perfroms well on Best-of-N sampling.}
\vspace{-5pt}
\resizebox{\textwidth}{!}{
\begin{tabular}{@{}ll|ccc|ccc|ccc|c@{}}
\toprule
\multicolumn{2}{l|}{Setup}                                      & \multicolumn{3}{c|}{Mistral-7B-Inst-v0.2} & \multicolumn{3}{c|}{Llama-3.1-8B-Inst} & \multicolumn{3}{c|}{Llama-3.1-70B-Inst} & \multirow{2}{*}{Avg.} \\ \cmidrule(r){1-11}
\multicolumn{1}{l|}{Train}                          & Inference & @4           & @16          & @64         & @4          & @16         & @64        & @4          & @16         & @64         &                       \\ \midrule
\multicolumn{1}{l|}{Llama-3.1-8B-Instruct}          & w/o Ref   & 14.8         & 16.2         & 18.4        & 49.0        & 50.4        & 52.2       & 69.6        & 71.0        & 71.0        & 45.8                  \\ \midrule
\multicolumn{1}{l|}{\multirow{2}{*}{+ DPO w/ Ref}}  & w/ Ref    & 18.6         & 24.4         & 28.8        & 54.0        & 55.4        & 57.0       & 71.8        & 71.2        & 72.2        & 50.4                  \\
\multicolumn{1}{l|}{}                               & w/o Ref   & 17.8         & 23.4         & 27.8        & 54.2        & 56.6        & 57.6       & 71.6        & 73.6        & 73.2        & 50.6                  \\ \midrule
\multicolumn{1}{l|}{\multirow{2}{*}{+ DPO w/o Ref}} & w/ Ref    & 17.8         & 23.4         & 28.4        & 54.0        & 55.2        & 57.6       & 70.6        & 72.0        & 73.2        & 50.2                  \\
\multicolumn{1}{l|}{}                               & w/o Ref   & 17.4         & 22.6         & 25.6        & 54.8        & 56.4        & 58.2       & 70.4        & 73.2        & 74.0        & 50.3                  \\ \bottomrule
\end{tabular}
}
\label{tab:ref}
\end{table}
\paragraph{Setup}
To this end, we explore two model training configurations: parameterizing the outcome reward either with or without a reference model. . 
We then apply both models to best-of-N sampling and evaluate whether including the reference model has any impact to the performance.
We also compare to directly using Llama-3.1-8B-Instruct, the reference model in our implicit PRM in previous experiments, as the reward model.
It serves as a controlled baseline without any RM training on our data, but has undergone preference learning~\citep{llama3modelcard}.

\paragraph{Results}
\looseness=-1
Surprisingly, no performance degradation is observed when the reference model is ablated in both training and inference, suggesting a more practically efficient variant of our approach.
Besides, Llama-3.1-8B-Instruct achieves strong performance too. 
This potentially explains why the reference model can be removed: The reference model is already capable of appropriately assigning high rewards to ``good'' steps and low ones to ``bad'' steps. 
Recall the process reward is $\sum_{i=t-1}^{t} \beta \log \pi_\theta(y_{i}|\mathbf{y}_{<i})/\pi_\text{ref}(y_{i}|\mathbf{y}_{<i})$.
Intuitively, a good step might receive high probabilities by both $\pi_\theta$ and $\pi_\text{ref}$, and therefore lowering its reward;
on the other hand, a bad step might receive low  probabilities by both, thereby increasing its reward.
This creates confusion to the PRM.
We argue that this behavior is actually beneficial during RL training: when the reference model $\pi_\text{ref}$ already performs well on certain actions, smaller rewards and consequently smaller policy gradients prevent over-training the policy model $\pi_\theta$ on these already-optimized actions.
Nevertheless, it is undesired on such inference-time response selection tasks.
This suggests that our implicit PRM is particularly appealing in practice, since most of the time practitioners will build their PRMs from a strong reference model such as Llama-3.1-8B-Instruct.
In such cases, $\pi_\text{ref}$ can be dropped in inference without hurting the performance as the above results suggest, and \textbf{our approach can achieve stronger performance than baselines with substantially cheaper training, without introducing any additional inference overhead.}

\section{Related Work}


\paragraph{Complex Reasoning of LLMs} Complex reasoning has become a key capability of Large Language Models (LLMs) yet remains challenging even to state-of-the-art ones \citep{Jimenez2023SWEbenchCL, Tian2024SciCodeAR}. 
Various techniques have been explored to improve LLMs on reasoning throughout different stages of their lifecycles, such as pre-training \citep{Azerbayev2023LlemmaAO, paster2023openwebmath, Li2023StarCoderMT}, post-training \citep{Luo2023WizardCoderEC, Yue2024MAmmoTH2SI, Yuan2024AdvancingLR, llama3modelcard, Ouyang2022TrainingLM}, and inference \citep{Wei2022ChainOT, Fu2022ComplexityBasedPF, Hao2023ReasoningWL, Lightman2023LetsVS}.
Among them, the process reward model (PRM) \citep{Lightman2023LetsVS}, which scores model outputs step by step, has attracted recent attention for its effectiveness in a variety of settings.


\paragraph{Implicit Reward}
Implicit reward has already been widely adopted in preference learning.
Despite primary work mainly focus on applying these algorithms to align models on top of supervised fine-tuning \citep{Rafailov2023DirectPO, Azar2023IPO, Ethayarajh2024KTOMA, Chen2024NoiseCA, Rosset2024DirectNO, Wu2024SelfPlayPO}, recent work also tries to leverage the implicit rewards of resulting models as outcome rewards \citep{Lambert2024RewardBenchER, Zhong2024DPOMP, Hosseini2024VSTaRTV}. Further, following \cite{Rafailov2024FromT}, which showed that DPO can automatically learn a Q function, \cite{Qiu2024TreeBoNEI} devise a self-guided decoding algorithm limited for DPO models leveraging such property. 
However, despite these applications of adopting DPO models as off-the-shelf reward models or Q functions, none of existing work specifically targets improving such ability or investigating how to derive decent PRMs upon those off-the-shelf models.

\section{Conclusion}
We start with a theoretical proposition demonstrating that parameterizing the outcome reward as the log-likelihood ratios of the policy and reference models $\log \frac{\pi_\theta(y)}{\pi_{\text{ref}}(y)}$, a PRM can be intrinsically learned at the same time without any extra training requirements.
We discuss its universality to instantiate different training objectives.
In experiments, we demonstrate that various implicit reward modeling objectives outperform baselines on MATH, with substantially better trade-offs between accuracy and development overhead, particularly the CE loss.
The performance of implicit PRMs can be further improved with majority voting. 
Further, scaling up instructions and responses benefit our implicit PRM, with the latter having a larger effect, but instructions should be relevant to downstream tasks while the diversity of responses does not bring gains.
Surprisingly, training on extra Math-Shepherd step labels brings no further improvements to our implicit PRM trained on only outcome data.

\bibliography{iclr2025_conference}

\begin{thebibliography}{47}
\providecommand{\natexlab}[1]{#1}
\providecommand{\url}[1]{\texttt{#1}}
\expandafter\ifx\csname urlstyle\endcsname\relax
  \providecommand{\doi}[1]{doi: #1}\else
  \providecommand{\doi}{doi: \begingroup \urlstyle{rm}\Url}\fi

\bibitem[Azar et~al.(2024)Azar, Rowland, Piot, Guo, Calandriello, Valko, and Munos]{Azar2023IPO}
Mohammad~Gheshlaghi Azar, Mark Rowland, Bilal Piot, Daniel Guo, Daniele Calandriello, Michal Valko, and R{\'e}mi Munos.
\newblock A general theoretical paradigm to understand learning from human preferences.
\newblock \emph{International Conference on Artificial Intelligence and Statistics}, abs/2310.12036, 2024.

\bibitem[Azerbayev et~al.(2024)Azerbayev, Schoelkopf, Paster, Santos, McAleer, Jiang, Deng, Biderman, and Welleck]{Azerbayev2023LlemmaAO}
Zhangir Azerbayev, Hailey Schoelkopf, Keiran Paster, Marco~Dos Santos, Stephen~Marcus McAleer, Albert~Q. Jiang, Jia Deng, Stella Biderman, and Sean Welleck.
\newblock Llemma: An open language model for mathematics.
\newblock \emph{ICLR}, 2024.

\bibitem[Cao et~al.(2024)Cao, Shu, Yu, Zhu, Wichers, Liu, and Meng]{Cao2024EnhancingRL}
Meng Cao, Lei Shu, Lei Yu, Yun Zhu, Nevan Wichers, Yinxiao Liu, and Lei Meng.
\newblock Enhancing reinforcement learning with dense rewards from language model critic.
\newblock In \emph{EMNLP}, 2024.

\bibitem[Chan et~al.(2024)Chan, Sun, Holt, and van~der Schaar]{Chan2024DenseRF}
Alex~J. Chan, Hao Sun, Samuel Holt, and Mihaela van~der Schaar.
\newblock Dense reward for free in reinforcement learning from human feedback.
\newblock \emph{ICML}, 2024.

\bibitem[Chen et~al.(2024)Chen, He, Yuan, Cui, Su, and Zhu]{Chen2024NoiseCA}
Huayu Chen, Guande He, Lifan Yuan, Ganqu Cui, Hang Su, and Jun Zhu.
\newblock Noise contrastive alignment of language models with explicit rewards.
\newblock \emph{ArXiv}, abs/2402.05369, 2024.

\bibitem[Cui et~al.(2024)Cui, Yuan, Ding, Yao, He, Zhu, Ni, Xie, Xie, Lin, Liu, and Sun]{Cui2023ULTRAFEEDBACKBL}
Ganqu Cui, Lifan Yuan, Ning Ding, Guanming Yao, Bingxiang He, Wei Zhu, Yuan Ni, Guotong Xie, Ruobing Xie, Yankai Lin, Zhiyuan Liu, and Maosong Sun.
\newblock Ultrafeedback: Boosting language models with scaled ai feedback.
\newblock In \emph{ICML}, 2024.

\bibitem[Ding et~al.(2023)Ding, Chen, Xu, Qin, Zheng, Hu, Liu, Sun, and Zhou]{ding2023enhancing}
Ning Ding, Yulin Chen, Bokai Xu, Yujia Qin, Zhi Zheng, Shengding Hu, Zhiyuan Liu, Maosong Sun, and Bowen Zhou.
\newblock Enhancing chat language models by scaling high-quality instructional conversations.
\newblock \emph{arXiv preprint arXiv:2305.14233}, 2023.

\bibitem[Ethayarajh et~al.(2024)Ethayarajh, Xu, Muennighoff, Jurafsky, and Kiela]{Ethayarajh2024KTOMA}
Kawin Ethayarajh, Winnie Xu, Niklas Muennighoff, Dan Jurafsky, and Douwe Kiela.
\newblock Kto: Model alignment as prospect theoretic optimization.
\newblock \emph{ICML}, 2024.

\bibitem[Fu et~al.(2023)Fu, Peng, Sabharwal, Clark, and Khot]{Fu2022ComplexityBasedPF}
Yao Fu, Hao Peng, Ashish Sabharwal, Peter Clark, and Tushar Khot.
\newblock Complexity-based prompting for multi-step reasoning.
\newblock \emph{ICLR}, 2023.

\bibitem[Gugger et~al.(2022)Gugger, Debut, Wolf, Schmid, Mueller, Mangrulkar, Sun, and Bossan]{accelerate}
Sylvain Gugger, Lysandre Debut, Thomas Wolf, Philipp Schmid, Zachary Mueller, Sourab Mangrulkar, Marc Sun, and Benjamin Bossan.
\newblock Accelerate: Training and inference at scale made simple, efficient and adaptable.
\newblock \url{https://github.com/huggingface/accelerate}, 2022.

\bibitem[Hao et~al.(2023)Hao, Gu, Ma, Hong, Wang, Wang, and Hu]{Hao2023ReasoningWL}
Shibo Hao, Yi~Gu, Haodi Ma, Joshua~Jiahua Hong, Zhen Wang, Daisy~Zhe Wang, and Zhiting Hu.
\newblock Reasoning with language model is planning with world model.
\newblock \emph{EMNLP}, 2023.

\bibitem[Hendrycks et~al.(2021)Hendrycks, Burns, Kadavath, Arora, Basart, Tang, Song, and Steinhardt]{Hendrycks2021MeasuringMP}
Dan Hendrycks, Collin Burns, Saurav Kadavath, Akul Arora, Steven Basart, Eric Tang, Dawn~Xiaodong Song, and Jacob Steinhardt.
\newblock Measuring mathematical problem solving with the math dataset.
\newblock \emph{ArXiv}, 2021.

\bibitem[Hosseini et~al.(2024)Hosseini, Yuan, Malkin, Courville, Sordoni, and Agarwal]{Hosseini2024VSTaRTV}
Arian Hosseini, Xingdi Yuan, Nikolay Malkin, Aaron~C. Courville, Alessandro Sordoni, and Rishabh Agarwal.
\newblock V-star: Training verifiers for self-taught reasoners.
\newblock \emph{COLM}, 2024.

\bibitem[Jiang et~al.(2023)Jiang, Sablayrolles, Mensch, Bamford, Chaplot, de~Las~Casas, Bressand, Lengyel, Lample, Saulnier, Lavaud, Lachaux, Stock, Scao, Lavril, Wang, Lacroix, and Sayed]{Jiang2023Mistral7}
Albert~Qiaochu Jiang, Alexandre Sablayrolles, Arthur Mensch, Chris Bamford, Devendra~Singh Chaplot, Diego de~Las~Casas, Florian Bressand, Gianna Lengyel, Guillaume Lample, Lucile Saulnier, L'elio~Renard Lavaud, Marie-Anne Lachaux, Pierre Stock, Teven~Le Scao, Thibaut Lavril, Thomas Wang, Timoth{\'e}e Lacroix, and William~El Sayed.
\newblock Mistral 7b.
\newblock \emph{ArXiv}, abs/2310.06825, 2023.

\bibitem[Jimenez et~al.(2024)Jimenez, Yang, Wettig, Yao, Pei, Press, and Narasimhan]{Jimenez2023SWEbenchCL}
Carlos~E. Jimenez, John Yang, Alexander Wettig, Shunyu Yao, Kexin Pei, Ofir Press, and Karthik Narasimhan.
\newblock Swe-bench: Can language models resolve real-world github issues?
\newblock \emph{ICLR}, 2024.

\bibitem[Kwon et~al.(2023)Kwon, Li, Zhuang, Sheng, Zheng, Yu, Gonzalez, Zhang, and Stoica]{kwon2023efficient}
Woosuk Kwon, Zhuohan Li, Siyuan Zhuang, Ying Sheng, Lianmin Zheng, Cody~Hao Yu, Joseph~E. Gonzalez, Hao Zhang, and Ion Stoica.
\newblock Efficient memory management for large language model serving with pagedattention.
\newblock In \emph{Proceedings of the ACM SIGOPS 29th Symposium on Operating Systems Principles}, 2023.

\bibitem[Lambert et~al.(2024)Lambert, Pyatkin, Morrison, Miranda, Lin, Chandu, Dziri, Kumar, Zick, Choi, Smith, and Hajishirzi]{Lambert2024RewardBenchER}
Nathan Lambert, Valentina Pyatkin, Jacob~Daniel Morrison, Lester James~Validad Miranda, Bill~Yuchen Lin, Khyathi~Raghavi Chandu, Nouha Dziri, Sachin Kumar, Tom Zick, Yejin Choi, Noah~A. Smith, and Hanna Hajishirzi.
\newblock Rewardbench: Evaluating reward models for language modeling.
\newblock \emph{ArXiv}, abs/2403.13787, 2024.

\bibitem[Leike(2024)]{Leike2024tweet}
Jan Leike, 2024.
\newblock URL \url{https://x.com/janleike/status/1821940180032594393?s=46}.

\bibitem[Li et~al.(2023)Li, Allal, Zi, Muennighoff, Kocetkov, Mou, Marone, Akiki, Li, Chim, Liu, Zheltonozhskii, Zhuo, Wang, Dehaene, Davaadorj, Lamy-Poirier, Monteiro, Shliazhko, Gontier, Meade, Zebaze, Yee, Umapathi, Zhu, Lipkin, Oblokulov, Wang, Murthy, Stillerman, Patel, Abulkhanov, Zocca, Dey, Zhang, Fahmy, Bhattacharyya, Yu, Singh, Luccioni, Villegas, Kunakov, Zhdanov, Romero, Lee, Timor, Ding, Schlesinger, Schoelkopf, Ebert, Dao, Mishra, Gu, Robinson, Anderson, Dolan-Gavitt, Contractor, Reddy, Fried, Bahdanau, Jernite, Ferrandis, Hughes, Wolf, Guha, von Werra, and de~Vries]{Li2023StarCoderMT}
Raymond Li, Loubna~Ben Allal, Yangtian Zi, Niklas Muennighoff, Denis Kocetkov, Chenghao Mou, Marc Marone, Christopher Akiki, Jia Li, Jenny Chim, Qian Liu, Evgenii Zheltonozhskii, Terry~Yue Zhuo, Thomas Wang, Olivier Dehaene, Mishig Davaadorj, Joel Lamy-Poirier, Jo{\~a}o Monteiro, Oleh Shliazhko, Nicolas Gontier, Nicholas Meade, Armel~Randy Zebaze, Ming-Ho Yee, Logesh~Kumar Umapathi, Jian Zhu, Benjamin Lipkin, Muhtasham Oblokulov, Zhiruo Wang, Rudra Murthy, Jason Stillerman, Siva~Sankalp Patel, Dmitry Abulkhanov, Marco Zocca, Manan Dey, Zhihan Zhang, Nourhan Fahmy, Urvashi Bhattacharyya, W.~Yu, Swayam Singh, Sasha Luccioni, Paulo Villegas, Maxim Kunakov, Fedor Zhdanov, Manuel Romero, Tony Lee, Nadav Timor, Jennifer Ding, Claire Schlesinger, Hailey Schoelkopf, Jana Ebert, Tri Dao, Mayank Mishra, Alexander Gu, Jennifer Robinson, Carolyn~Jane Anderson, Brendan Dolan-Gavitt, Danish Contractor, Siva Reddy, Daniel Fried, Dzmitry Bahdanau, Yacine Jernite, Carlos~Mu{\~n}oz Ferrandis, Sean~M. Hughes, Thomas Wolf, Arjun
  Guha, Leandro von Werra, and Harm de~Vries.
\newblock Starcoder: may the source be with you!
\newblock \emph{TMLR}, 2023.

\bibitem[Lightman et~al.(2023)Lightman, Kosaraju, Burda, Edwards, Baker, Lee, Leike, Schulman, Sutskever, and Cobbe]{Lightman2023LetsVS}
Hunter Lightman, Vineet Kosaraju, Yura Burda, Harrison Edwards, Bowen Baker, Teddy Lee, Jan Leike, John Schulman, Ilya Sutskever, and Karl Cobbe.
\newblock Let's verify step by step.
\newblock \emph{ArXiv}, 2023.

\bibitem[Liu et~al.(2024)Liu, Zeng, Liu, Yan, He, Wang, Yan, Liu, and Zhou]{liu2024skywork}
Chris~Yuhao Liu, Liang Zeng, Jiacai Liu, Rui Yan, Jujie He, Chaojie Wang, Shuicheng Yan, Yang Liu, and Yahui Zhou.
\newblock Skywork-reward: Bag of tricks for reward modeling in llms.
\newblock \emph{arXiv preprint arXiv:2410.18451}, 2024.

\bibitem[Lu et~al.(2024)Lu, Dou, Wang, Cao, Dai, Wan, Huang, and Guo]{Lu2024AutoCVER}
Jianqiao Lu, Zhiyang Dou, Hongru Wang, Zeyu Cao, Jianbo Dai, Yingjia Wan, Yinya Huang, and Zhijiang Guo.
\newblock Autopsv: Automated process-supervised verifier.
\newblock \emph{ArXiv}, abs/2405.16802, 2024.

\bibitem[Luo et~al.(2024{\natexlab{a}})Luo, Liu, Liu, Phatale, Lara, Li, Shu, Zhu, Meng, Sun, and Rastogi]{Luo2024ImproveMR}
Liangchen Luo, Yinxiao Liu, Rosanne Liu, Samrat Phatale, Harsh Lara, Yunxuan Li, Lei Shu, Yun Zhu, Lei Meng, Jiao Sun, and Abhinav Rastogi.
\newblock Improve mathematical reasoning in language models by automated process supervision.
\newblock \emph{ArXiv}, abs/2406.06592, 2024{\natexlab{a}}.

\bibitem[Luo et~al.(2024{\natexlab{b}})Luo, Xu, Zhao, Sun, Geng, Hu, Tao, Ma, Lin, and Jiang]{Luo2023WizardCoderEC}
Ziyang Luo, Can Xu, Pu~Zhao, Qingfeng Sun, Xiubo Geng, Wenxiang Hu, Chongyang Tao, Jing Ma, Qingwei Lin, and Daxin Jiang.
\newblock Wizardcoder: Empowering code large language models with evol-instruct.
\newblock \emph{ICLR}, 2024{\natexlab{b}}.

\bibitem[Mahan et~al.(2024)Mahan, Phung, Rafailov, Blagden, Lile, Castricato, Franken, Finn, and Albalak]{Mahan2024GenerativeRM}
Dakota Mahan, Duy Phung, Rafael Rafailov, Chase Blagden, Nathan Lile, Louis Castricato, Jan-Philipp Franken, Chelsea Finn, and Alon Albalak.
\newblock Generative reward models.
\newblock 2024.

\bibitem[Meta(2024)]{llama3modelcard}
Meta.
\newblock Llama 3 model card.
\newblock \emph{Github}, 2024.
\newblock URL \url{https://github.com/meta-llama/llama3/blob/main/MODEL_CARD.md}.

\bibitem[Ouyang et~al.(2022)Ouyang, Wu, Jiang, Almeida, Wainwright, Mishkin, Zhang, Agarwal, Slama, Ray, Schulman, Hilton, Kelton, Miller, Simens, Askell, Welinder, Christiano, Leike, and Lowe]{Ouyang2022TrainingLM}
Long Ouyang, Jeff Wu, Xu~Jiang, Diogo Almeida, Carroll~L. Wainwright, Pamela Mishkin, Chong Zhang, Sandhini Agarwal, Katarina Slama, Alex Ray, John Schulman, Jacob Hilton, Fraser Kelton, Luke~E. Miller, Maddie Simens, Amanda Askell, Peter Welinder, Paul~Francis Christiano, Jan Leike, and Ryan~J. Lowe.
\newblock Training language models to follow instructions with human feedback.
\newblock \emph{ArXiv}, abs/2203.02155, 2022.

\bibitem[Paster et~al.(2024)Paster, Santos, Azerbayev, and Ba]{paster2023openwebmath}
Keiran Paster, Marco~Dos Santos, Zhangir Azerbayev, and Jimmy Ba.
\newblock Openwebmath: An open dataset of high-quality mathematical web text, 2024.

\bibitem[Qiu et~al.(2024)Qiu, Lu, Zeng, Guo, Geng, Wang, Huang, Wu, and Wang]{Qiu2024TreeBoNEI}
Jiahao Qiu, Yifu Lu, Yifan Zeng, Jiacheng Guo, Jiayi Geng, Huazheng Wang, Kaixuan Huang, Yue Wu, and Mengdi Wang.
\newblock Treebon: Enhancing inference-time alignment with speculative tree-search and best-of-n sampling.
\newblock 2024.

\bibitem[Rafailov et~al.(2023)Rafailov, Sharma, Mitchell, Ermon, Manning, and Finn]{Rafailov2023DirectPO}
Rafael Rafailov, Archit Sharma, Eric Mitchell, Stefano Ermon, Christopher~D. Manning, and Chelsea Finn.
\newblock Direct preference optimization: Your language model is secretly a reward model.
\newblock \emph{NeurIPS}, 2023.

\bibitem[Rafailov et~al.(2024)Rafailov, Hejna, Park, and Finn]{Rafailov2024FromT}
Rafael Rafailov, Joey Hejna, Ryan Park, and Chelsea Finn.
\newblock From $r$ to $q^*$: Your language model is secretly a q-function.
\newblock \emph{ArXiv}, 2024.

\bibitem[Rosset et~al.(2024)Rosset, Cheng, Mitra, Santacroce, Awadallah, and Xie]{Rosset2024DirectNO}
Corby Rosset, Ching-An Cheng, Arindam Mitra, Michael Santacroce, Ahmed Awadallah, and Tengyang Xie.
\newblock Direct nash optimization: Teaching language models to self-improve with general preferences.
\newblock \emph{ArXiv}, abs/2404.03715, 2024.

\bibitem[Setlur et~al.(2024)Setlur, Nagpal, Fisch, Geng, Eisenstein, Agarwal, Agarwal, Berant, and Kumar]{Setlur2024RewardingPS}
Amrith~Rajagopal Setlur, Chirag Nagpal, Adam Fisch, Xinyang Geng, Jacob Eisenstein, Rishabh Agarwal, Alekh Agarwal, Jonathan Berant, and Aviral Kumar.
\newblock Rewarding progress: Scaling automated process verifiers for llm reasoning.
\newblock 2024.

\bibitem[Snell et~al.(2024)Snell, Lee, Xu, and Kumar]{Snell2024ScalingLT}
Charlie Snell, Jaehoon Lee, Kelvin Xu, and Aviral Kumar.
\newblock Scaling llm test-time compute optimally can be more effective than scaling model parameters.
\newblock \emph{ArXiv}, abs/2408.03314, 2024.

\bibitem[Tian et~al.(2024)Tian, Gao, Zhang, Chen, Fan, Guo, Haas, Ji, Krongchon, Li, Liu, Luo, Ma, Tong, Trinh, Tian, Wang, Wu, Xiong, Yin, Zhu, Lieret, Lu, Liu, Du, Tao, Press, Callan, Huerta, and Peng]{Tian2024SciCodeAR}
Minyang Tian, Luyu Gao, Shizhuo~Dylan Zhang, Xinan Chen, Cunwei Fan, Xuefei Guo, Roland Haas, Pan Ji, Kittithat Krongchon, Yao Li, Shengyan Liu, Di~Luo, Yutao Ma, Hao Tong, Kha Trinh, Chenyu Tian, Zihan Wang, Bohao Wu, Yanyu Xiong, Shengzhu Yin, Min Zhu, Kilian Lieret, Yanxin Lu, Genglin Liu, Yufeng Du, Tianhua Tao, Ofir Press, Jamie Callan, E.~A. Huerta, and Hao Peng.
\newblock Scicode: A research coding benchmark curated by scientists.
\newblock \emph{Arxiv}, 2024.

\bibitem[Touvron et~al.(2023)Touvron, Martin, Stone, Albert, Almahairi, Babaei, Bashlykov, Batra, Bhargava, Bhosale, Bikel, Blecher, Ferrer, Chen, Cucurull, Esiobu, Fernandes, Fu, Fu, Fuller, Gao, Goswami, Goyal, Hartshorn, Hosseini, Hou, Inan, Kardas, Kerkez, Khabsa, Kloumann, Korenev, Koura, Lachaux, Lavril, Lee, Liskovich, Lu, Mao, Martinet, Mihaylov, Mishra, Molybog, Nie, Poulton, Reizenstein, Rungta, Saladi, Schelten, Silva, Smith, Subramanian, Tan, Tang, Taylor, Williams, Kuan, Xu, Yan, Zarov, Zhang, Fan, Kambadur, Narang, Rodriguez, Stojnic, Edunov, and Scialom]{Touvron2023Llama2O}
Hugo Touvron, Louis Martin, Kevin~R. Stone, Peter Albert, Amjad Almahairi, Yasmine Babaei, Nikolay Bashlykov, Soumya Batra, Prajjwal Bhargava, Shruti Bhosale, Daniel~M. Bikel, Lukas Blecher, Cristian~Cant{\'o}n Ferrer, Moya Chen, Guillem Cucurull, David Esiobu, Jude Fernandes, Jeremy Fu, Wenyin Fu, Brian Fuller, Cynthia Gao, Vedanuj Goswami, Naman Goyal, Anthony~S. Hartshorn, Saghar Hosseini, Rui Hou, Hakan Inan, Marcin Kardas, Viktor Kerkez, Madian Khabsa, Isabel~M. Kloumann, A.~V. Korenev, Punit~Singh Koura, Marie-Anne Lachaux, Thibaut Lavril, Jenya Lee, Diana Liskovich, Yinghai Lu, Yuning Mao, Xavier Martinet, Todor Mihaylov, Pushkar Mishra, Igor Molybog, Yixin Nie, Andrew Poulton, Jeremy Reizenstein, Rashi Rungta, Kalyan Saladi, Alan Schelten, Ruan Silva, Eric~Michael Smith, R.~Subramanian, Xia Tan, Binh Tang, Ross Taylor, Adina Williams, Jian~Xiang Kuan, Puxin Xu, Zhengxu Yan, Iliyan Zarov, Yuchen Zhang, Angela Fan, Melanie Kambadur, Sharan Narang, Aurelien Rodriguez, Robert Stojnic, Sergey Edunov, and
  Thomas Scialom.
\newblock Llama 2: Open foundation and fine-tuned chat models.
\newblock \emph{ArXiv}, abs/2307.09288, 2023.

\bibitem[Wang et~al.(2024)Wang, Xiong, Xie, Zhao, and Zhang]{ArmoRM}
Haoxiang Wang, Wei Xiong, Tengyang Xie, Han Zhao, and Tong Zhang.
\newblock Interpretable preferences via multi-objective reward modeling and mixture-of-experts.
\newblock In \emph{EMNLP}, 2024.

\bibitem[Wang et~al.(2023)Wang, Li, Shao, Xu, Dai, Li, Chen, Y.Wu, and Sui]{Wang2023MathShepherdVA}
Peiyi Wang, Lei Li, Zhihong Shao, Runxin Xu, Damai Dai, Yifei Li, Deli Chen, Y.Wu, and Zhifang Sui.
\newblock Math-shepherd: Verify and reinforce llms step-by-step without human annotations.
\newblock \emph{ArXiv}, 2023.

\bibitem[Wei et~al.(2022)Wei, Wang, Schuurmans, Bosma, hsin Chi, Xia, Le, and Zhou]{Wei2022ChainOT}
Jason Wei, Xuezhi Wang, Dale Schuurmans, Maarten Bosma, Ed~Huai hsin Chi, F.~Xia, Quoc Le, and Denny Zhou.
\newblock Chain of thought prompting elicits reasoning in large language models.
\newblock \emph{NeurIPS}, 2022.

\bibitem[Wu et~al.(2024)Wu, Sun, Yuan, Ji, Yang, and Gu]{Wu2024SelfPlayPO}
Yue Wu, Zhiqing Sun, Huizhuo Yuan, Kaixuan Ji, Yiming Yang, and Quanquan Gu.
\newblock Self-play preference optimization for language model alignment.
\newblock \emph{ArXiv}, abs/2405.00675, 2024.

\bibitem[Yuan et~al.(2024)Yuan, Cui, Wang, Ding, Wang, Deng, Shan, Chen, Xie, Lin, Liu, Zhou, Peng, Liu, and Sun]{Yuan2024AdvancingLR}
Lifan Yuan, Ganqu Cui, Hanbin Wang, Ning Ding, Xingyao Wang, Jia Deng, Boji Shan, Huimin Chen, Ruobing Xie, Yankai Lin, Zhenghao Liu, Bowen Zhou, Hao Peng, Zhiyuan Liu, and Maosong Sun.
\newblock Advancing llm reasoning generalists with preference trees.
\newblock \emph{ArXiv}, 2024.

\bibitem[Yue et~al.(2024)Yue, Zheng, Zhang, and Chen]{Yue2024MAmmoTH2SI}
Xiang Yue, Tuney Zheng, Ge~Zhang, and Wenhu Chen.
\newblock Mammoth2: Scaling instructions from the web.
\newblock \emph{NeurIPS}, 2024.

\bibitem[Zhang et~al.(2024{\natexlab{a}})Zhang, Wang, Diao, Lin, Pan, Dong, Zhang, Molchanov, and Zhang]{Zhang2024ERPRM}
Hanning Zhang, Pengcheng Wang, Shizhe Diao, Yong Lin, Rui Pan, Hanze Dong, Dylan Zhang, Pavlo Molchanov, and Tong Zhang.
\newblock Entropy-regularized process reward model, 2024{\natexlab{a}}.

\bibitem[Zhang et~al.(2024{\natexlab{b}})Zhang, Zeng, Hua, Ding, Chen, Ma, Li, Cui, Qi, Zhu, et~al.]{zhang2024ultramedical}
Kaiyan Zhang, Sihang Zeng, Ermo Hua, Ning Ding, Zhang-Ren Chen, Zhiyuan Ma, Haoxin Li, Ganqu Cui, Biqing Qi, Xuekai Zhu, et~al.
\newblock Ultramedical: Building specialized generalists in biomedicine.
\newblock \emph{arXiv preprint arXiv:2406.03949}, 2024{\natexlab{b}}.

\bibitem[Zhang et~al.(2024{\natexlab{c}})Zhang, Hosseini, Bansal, Kazemi, Kumar, and Agarwal]{Zhang2024GenerativeVR}
Lunjun Zhang, Arian Hosseini, Hritik Bansal, Mehran Kazemi, Aviral Kumar, and Rishabh Agarwal.
\newblock Generative verifiers: Reward modeling as next-token prediction.
\newblock 2024{\natexlab{c}}.

\bibitem[Zhong et~al.(2024)Zhong, Feng, Xiong, Zhao, He, Bian, and Wang]{Zhong2024DPOMP}
Han Zhong, Guhao Feng, Wei Xiong, Li~Zhao, Di~He, Jiang Bian, and Liwei Wang.
\newblock Dpo meets ppo: Reinforced token optimization for rlhf.
\newblock \emph{ArXiv}, abs/2404.18922, 2024.

\bibitem[Zhu et~al.(2023)Zhu, Frick, Wu, Zhu, and Jiao]{starling2023}
Banghua Zhu, Evan Frick, Tianhao Wu, Hanlin Zhu, and Jiantao Jiao.
\newblock Starling-7b: Improving llm helpfulness \& harmlessness with rlaif, November 2023.

\end{thebibliography}
\bibliographystyle{iclr2025_conference}

\appendix
\section{Proof of Proposition}

\label{sec:appendix_proof}

\begin{proposition}\label{proposition_proof_1} Consider an ORM where the reward is parameterized by the log-likelihood ratio of two causal LMs, i.e. $r_\theta(\mathbf{y}):= \beta \log \frac{\pi_\theta(\mathbf{y})}{\pi_\text{ref}(\mathbf{y})}$. 
Define $q_\theta^t(\mathbf{y}_{<t}, y_t):= \sum_{i=1}^{t} \beta \log \frac{\pi_\theta(y_{i}|\mathbf{y}_{<i})}{\pi_\text{ref}(y_{i}|\mathbf{y}_{<i})}$. 
$q_\theta^t$ is the exponential average of $r_\theta$ at step $t$.
\begin{equation}
q_\theta^{t}(\mathbf{y}_{<t}, y_t) = \beta \log \E_{\pi_\text{ref}(\mathbf{y}|\mathbf{y}_{\leq t})} e^{\frac{1}{\beta}r_\theta(\mathbf{y})}
\end{equation}

\end{proposition}
\begin{proof}
The Proposition can be proven using mathematical induction.

Suppose response $\mathbf{y}$ has $T$ tokens.

    \textbf{(1)} For $\forall t < T$, if $q_\theta^{t+1}(\mathbf{y}_{<t+1}, y_{t+1}) = \beta \log \E_{\pi_\text{ref}(\mathbf{y}|\mathbf{y}_{\leq {t+1}})} e^{\frac{1}{\beta}r_\theta(\mathbf{y})}$ holds, then $q_\theta^t(\mathbf{y}_{<t}, y_t) = \beta \log \E_{\pi_\text{ref}(\mathbf{y}|\mathbf{y}_{\leq t})} e^{\frac{1}{\beta}r_\theta(\mathbf{y})}$ would also hold.

    \textbf{(2)} At $t=T$, $q_\theta^T(\mathbf{y}_{<T}, y_T) = r_\theta(\mathbf{y}) = \beta \log \E_{\pi_\text{ref}(\mathbf{y}|\mathbf{y}_{\leq T})} e^{\frac{1}{\beta}r_\theta(\mathbf{y})}$.

\textbf{proof of (1):}

\begin{align*}
    \beta \log \E_{\pi_\text{ref}(\mathbf{y}|\mathbf{y}_{\leq t})} e^{\frac{1}{\beta}r_\theta(\mathbf{y})}&= \beta \log \E_{\pi_\text{ref}(\mathbf{y_{t+1}}|\mathbf{y}_{\leq t})} \E_{\pi_\text{ref}(\mathbf{y}|\mathbf{y}_{\leq t+1})} e^{\frac{1}{\beta}r_\theta(\mathbf{y})}  \\
&= \beta \log \E_{\pi_\text{ref}(\mathbf{y_{t+1}}|\mathbf{y}_{\leq t})}  e^{\frac{1}{\beta}q_\theta^{t+1}(\mathbf{y}_{<t+1}, y_{t+1})} \\
&= \beta \log \E_{\pi_\text{ref}(\mathbf{y_{t+1}}|\mathbf{y}_{\leq t})}  \prod_{i=1}^{t+1}   \frac{\pi_\theta(y_{i}|\mathbf{y}_{<i})}{\pi_\text{ref}(y_{i}|\mathbf{y}_{<i})} \\
&= \beta \log \prod_{i=1}^{t}   \frac{\pi_\theta(y_{i}|\mathbf{y}_{<i})}{\pi_\text{ref}(y_{i}|\mathbf{y}_{<i})} \E_{\pi_\text{ref}(\mathbf{y_{t+1}}|\mathbf{y}_{\leq t})}    \frac{\pi_\theta(y_{t+1}|\mathbf{y}_{\leq t})}{\pi_\text{ref}(y_{t+1}|\mathbf{y}_{\leq t})} \\
&= \beta \log \prod_{i=1}^{t}   \frac{\pi_\theta(y_{i}|\mathbf{y}_{<i})}{\pi_\text{ref}(y_{i}|\mathbf{y}_{<i})} \sum_{y_{t+1}}{\pi_\text{ref}(y_{t+1}|\mathbf{y}_{\leq t})    \frac{\pi_\theta(y_{t+1}|\mathbf{y}_{\leq t})}{\pi_\text{ref}(y_{t+1}|\mathbf{y}_{\leq t})}} \\
&= \beta \log \prod_{i=1}^{t}   \frac{\pi_\theta(y_{i}|\mathbf{y}_{<i})}{\pi_\text{ref}(y_{i}|\mathbf{y}_{<i})} \sum_{y_{t+1}}{\pi_\theta(y_{t+1}|\mathbf{y}_{\leq t})} \\
&= \beta \log \prod_{i=1}^{t}   \frac{\pi_\theta(y_{i}|\mathbf{y}_{<i})}{\pi_\text{ref}(y_{i}|\mathbf{y}_{<i})}\\
\end{align*}

\textbf{proof of (2):}

The conclusion is straightforward. Since $\pi$ is autoregressive, we have
\begin{equation*}
    r_\theta(\mathbf{y}) := \beta \log \frac{\pi_\theta(\mathbf{y})}{\pi_\text{ref}(\mathbf{y})} =  \beta \log  \prod_{i=1}^{T} \frac{\pi_\theta(y_{i}|\mathbf{y}_{<i})}{\pi_\text{ref}(y_{i}|\mathbf{y}_{<i})} = \sum_{i=1}^{T} \beta \log   \frac{\pi_\theta(y_{i}|\mathbf{y}_{<i})}{\pi_\text{ref}(y_{i}|\mathbf{y}_{<i})}.
\end{equation*}
Since $\mathbf{y}_{\leq T} = \mathbf{y}$, the expectation $\E_{\pi_\text{ref}(\mathbf{y}|\mathbf{y}_{\leq T})}$ can be removed:
\begin{equation*}
    \beta \log \E_{\pi_\text{ref}(\mathbf{y}|\mathbf{y}_{\leq T})} e^{\frac{1}{\beta}r_\theta(\mathbf{y})}=\beta \log e^{\frac{1}{\beta}r_\theta(\mathbf{y})} = r_\theta(\mathbf{y}).
\end{equation*}
\end{proof}

\end{document}